\documentclass[10pt, conference, letterpaper]{IEEEtran}
\IEEEoverridecommandlockouts
\usepackage{graphicx}
\usepackage{epstopdf}
\usepackage{cite}
\usepackage{amsmath,amssymb,amsfonts}
\usepackage{algorithmic}
\usepackage{graphicx}
\usepackage{textcomp}
\usepackage{xcolor}
\usepackage{enumitem,kantlipsum}
\usepackage{comment}
\usepackage{url}
\usepackage{amsthm}
\usepackage{balance}

\usepackage{caption}
\usepackage{subcaption}
\usepackage{multirow}
\usepackage{setspace}

\usepackage{mathtools}

\newtheorem{thm}{Theorem}
\newtheorem{lem}{Lemma}
\newtheorem{cor}{Corollary}
\theoremstyle{remark}

\theoremstyle{definition}

\newtheorem{asmp}{Assumption}

\usepackage{fixltx2e}
\usepackage{accents}
\MakeRobust{\underaccent}

\newcommand{\carlee}[1]{{\bf \color{red} Carlee: #1}}

\newcommand{\ubar}[1]{\underaccent{\bar}{#1}}
\newcommand{\ie}{{\em i.e.}}
\newcommand{\eg}{{\em e.g.}}

\newcommand{\expect}{\mathbb{E}\expectarg}
\DeclarePairedDelimiterX{\expectarg}[1]{[}{]}{%
	\ifnum\currentgrouptype=16 \else\begingroup\fi
	\activatebar#1
	\ifnum\currentgrouptype=16 \else\endgroup\fi
}
\newcommand{\innermid}{\nonscript\;\delimsize\vert\nonscript\;}
\newcommand{\activatebar}{%
	\begingroup\lccode`\~=`\|
	\lowercase{\endgroup\let~}\innermid 
	\mathcode`|=\string"8000
}
\usepackage{mathtools}
\DeclarePairedDelimiter\ceil{\lceil}{\rceil}
\DeclarePairedDelimiter\floor{\lfloor}{\rfloor}
\DeclareMathOperator*{\argmin}{arg\,min}

\usepackage[short]{optional}
\usepackage{bbm}
\usepackage{comment}
\newtheorem{remark}{Remark}

\def\BibTeX{{\rm B\kern-.05em{\sc i\kern-.025em b}\kern-.08em
    T\kern-.1667em\lower.7ex\hbox{E}\kern-.125emX}}
    
\begin{document}

\title{Machine Learning on Volatile Instances\\
}

\author{\IEEEauthorblockN{
        Xiaoxi Zhang, Jianyu Wang, Gauri Joshi, Carlee Joe-Wong}
        \IEEEauthorblockA{
        Department of Electrical and Computer Engineering, Carnegie Mellon University, USA
        }
        \IEEEauthorblockA{
        Email:\{xiaoxiz2, jianyuw1, gaurij, cjoewong\}@andrew.cmu.edu
        }
}

\maketitle

\begin{abstract}
Due to the massive size of the neural network models and training datasets used in machine learning today, it is imperative to distribute stochastic gradient descent (SGD) by splitting up tasks such as gradient evaluation across multiple worker nodes. However, running distributed SGD can be prohibitively expensive because it may require specialized computing resources such as GPUs for extended periods of time. We propose cost-effective strategies to exploit volatile cloud instances that are cheaper than standard instances, but may be interrupted by higher priority workloads. To the best of our knowledge, this work is the first to quantify how variations in the number of active worker nodes (as a result of preemption) affects SGD convergence and the time to train the model. By understanding these trade-offs between preemption probability of the instances, accuracy, and training time, we are able to derive practical strategies for configuring distributed SGD jobs on volatile instances such as Amazon EC2 spot instances and other preemptible cloud instances. 
Experimental results show that our strategies achieve good training performance at substantially lower cost. 
\end{abstract}


\section{Introduction}
\label{sec:intro}


Stochastic gradient descent (SGD) is the core algorithm used by most state-of-the-art machine learning (ML) problems today~\cite{SGD1951,sgd-ls,parallel-training-ls-distbelief}. Yet as ever more complex models are trained on ever larger amounts of data, most SGD implementations have been forced to distribute the task of computing gradients across multiple ``worker'' nodes, thus reducing the computational burden on any single node while speeding up the model training through parallelization. Currently, even distributed training jobs require high-performance computing infrastructure such as GPUs to finish in a reasonable amount of time. However, purchasing GPUs outright is expensive and requires intensive setup and maintenance. Renting such machines as on-demand instances from services like Amazon EC2 can reduce setup costs, but may still be prohibitively expensive since distributed training jobs can take hours or even days to complete.

A common way to save money on cloud instances is to utilize \emph{volatile}, or \emph{transient}, instances, which have lower prices but experience interruptions~\cite{spot,gcp,low_priority_azure}. Examples of such instances include Google Cloud Platform's preemptible instances~\cite{gcp} and Azure's low-priority virtual machines~\cite{low_priority_azure}; both give users access to virtual machines that can be preempted at any time, but charge a significantly lower hourly price than on-demand instances with availability guarantees. Amazon EC2's spot instances offer a similar service, but provide users additional flexibility by dynamically changing the price charged for using spot instances. Users can then specify the maximum price they are willing to pay, and they do not receive access to the instance when the prevailing spot price exceeds their specified maximum price~\cite{aws-fleet}. 
Volatile computing resources may also be used to train ML jobs outside of traditional cloud contexts, e.g., in datacenters that run on ``stranded power.'' Such datacenters only activate instances when the energy network supplying power to the datacenter has excess energy that needs to be burned off~\cite{yang2016zccloud,chien2016characterizing}, leading to significant temporal volatility in resource availability. SGD variants are also commonly used to train machine learning models in edge computing contexts, where resource volatility is a significant practical challenge~\cite{konevcny2016federated,tao2018esgd}.

SGD algorithms can be run on volatile instances by deploying each worker on a single instance, and deploying a parameter server on an on-demand or reserved instance that is never interrupted~\cite{BidCloud}. This deployment strategy, however, has drawbacks: since the workers may be interrupted throughout the training process, they cannot update the model parameters as frequently, increasing the error of the trained model compared to deploying workers on on-demand instances. Compensating for this increased error would require either training the model for a larger number of iterations or increasing the number of provisioned workers, both of which will increase the training cost. In this paper, we \emph{quantify the performance tradeoffs between error, cost, and training time for volatile instances}. We then use our analysis to propose \emph{practical strategies for optimizing these tradeoffs} in realistic preemption environments. We first consider Amazon spot instances, for which users can indirectly control their preemptions by setting maximum bids, and derive the resulting optimal bidding strategies.
We then derive the optimal number of iterations and workers when users cannot control their instances' preemptions, as in GCP's preemptible instances and Azure's low-priority VMs.  
More specifically, this work makes the following contributions:
\begin{enumerate}[wide, labelwidth=!, labelindent=0pt]
\item \emph{Quantifying training error convergence with dynamic numbers of workers (Section~\ref{sec:error_runtime}).} 
Using volatile instances that can be interrupted and may rejoin later presents a new research challenge: prior analyses of distributed SGD algorithms do not consider the possibility that the number of active workers will change over time. We derive new error bounds on the convergence of SGD methods when the number of workers varies over time and show that the bound is proportional to the expected reciprocal of the number of active workers. 

\item \emph{Deriving optimal spot bidding strategies (Section~\ref{sec:spot-bid}).} 
To the best of our knowledge, no works have yet explored bidding strategies for distributed machine learning jobs that consider the bidding's effect on error convergence and random iteration runtimes. 
We analyze a unique three-way trade-off between the cost, error, and training time, using which we can design optimal bidding strategies to control the preemptions of spot instances. For tractability, we focus on the case where each worker submits one of two distinct bids. 

\item \emph{Deriving the optimal number of workers (Section~\ref{sec:preempt}).} For scenarios where users cannot control the preemption probability, we propose a general model to relate the number of provisioned workers to the expected reciprocal of the number of active workers, which can capture practical preemption distributions. Using this model, we then provide mathematical expressions to jointly optimize the number of provisioned workers and iterations. We also propose a strategy to dynamically adjust the number of provisioned workers, which can further improve the error convergence. 

\item \emph{Experimental validation on Amazon EC2 (Section~\ref{sec:experiments}).} We validate our results by running distributed SGD jobs analyzing the CIFAR-10~\cite{cifar10} dataset on Amazon EC2. We show that our derived optimal bid prices can reduce users' cost by 65\% on real, and 62\% on synthetic, spot price traces 
while meeting the same error and completion time requirements, compared with bidding a high price to minimize interruptions as suggested in~\cite{NotBidCloud}. Moreover, we implement and validate two simple but effective dynamic strategies that reduce the cost and yield a better cost/completion time/error trade-off: (i) adding workers later in the job and re-optimizing the bids according to the realized error and training time so far, and (ii) exponentially increasing the number of provisioned workers and running for a logarithmic number of iterations.
\end{enumerate}
\section{Related Work}
\label{sec:related}

Our work is broadly related to prior works on algorithm analysis for distributed machine learning, as well as exploiting spot instances to efficiently run computational jobs. 

{\bf Distributed machine learning} generally assumes that multiple workers send local computation results to be aggregated at a central server, which then sends them updated parameter values. The SGD algorithm~\cite{SGD1951}, in which workers compute the gradients of a given objective function with respect to model parameters, is particularly popular. In SGD, workers individually compute the gradient over stochastic samples (usually a mini-batch \cite{mini-batch}) chosen from data residing at each worker in each iteration. Recent work has attempted to limit device-server communication to reduce the training time of SGD and related models~\cite{konevcny2016federated,shamir2014communication,kamp2018efficient,McMahan2017Federated}, while others analyze the effect of the mini-batch size \cite{mini-batch} or learning rate \cite{Staleness:Aistats,Bottou} on SGD algorithms' training error. Bottou et al.~\cite{Bottou} analyze the convergence of training error in SGD but do not consider the runtime per iteration. Dutta et al.~\cite{Staleness:Aistats} analyze the trade-off between the training error and the (wall-clock) training time of distributed SGD, accounting for stochastic runtimes for the gradient computations at different workers~\cite{wang2019efficient}. Our work is similar in spirit but focuses on spot instances, which introduces cost as another performance metric.
%
We also go beyond~\cite{Bottou,Staleness:Aistats} to derive error bounds when the number of active workers changes in different iterations.

{\bf Utilizing spot and other transient cloud resources} for computing jobs has been extensively studied. Zheng et al.~\cite{BidCloud} design optimal bids to minimize the cost of completing jobs with a pre-determined execution time and no deadline. Other works derive cost-aware bidding strategies that consider jobs' deadline constraints \cite{spot-ddl} or jointly optimize the use of spot and on-demand instances \cite{proteus}. However, these frameworks cannot handle distributed SGD's dependencies between workers.
Another line of work instead optimizes the markets in which users bid for spot instances. Sharma et al.~\cite{NotBidCloud} advocate bidding the price of an on-demand instance and migrating to VM instances in other spot markets upon interruptions. The resulting migration overhead, however, requires complex checkpointing and migration strategies due to SGD's substantial communication dependencies between workers, realizing limited savings~\cite{lee2017deepspotcloud}. Some software frameworks have been designed for running big data analytics on transient instances~\cite{yan2016tr}, but they do not include theoretical ML performance analyses.
\section{Error and Runtime Analysis of Distributed SGD with Volatile Workers}\label{sec:error_runtime}

The number of active computing nodes used for distributed SGD training affects the convergence of the training error versus the number of SGD iterations as well as the runtime spent per iteration. Unlike most previous works in the optimization theory literature, which focus only on error-versus-iterations convergence, we consider both these factors and analyze the true convergence of SGD with respect to the wall-clock time. Moreover, to the best of our knowledge this is the first work that presents an error and runtime analysis for volatile computing instances, which can result in a changing number of active workers during training.

We formally introduce distributed SGD in Section~\ref{sec:sgd}. In Section~\ref{sec:error_analysis}, we quantify how the preemption probability adversely affects error convergence because having fewer active workers yields more noisy gradients. In Section~\ref{sec:runtime_analysis}, we analyze the effect of worker volatility on the training runtime, which is affected in two opposing ways. A higher preemption probability results in longer dead time intervals where we have zero active workers. Although a lower preemption probability yields more active workers, it can increase synchronization delays in waiting for straggling nodes. This error and runtime analysis lays the foundation for subsequent results on bidding strategies that can dynamically control the probability of preemption and the number of active worker nodes. 



In Sections~\ref{sec:spot-bid} and~\ref{sec:preempt}, we use our results on the error and runtime analysis from this section to minimize the cost of training a job, subject to constraints on the maximum allowable error and runtime. Our goal is to solve the optimization:
 \begin{align}\label{eq:simp-obj-cost}
 {\rm minimize:} \quad & \text{Expected total cost $\mathbb{E}[C]$}\\
 \text{st.:} \quad
 &	\text{Expected training error } \mathbb{E}[\phi]  \le \epsilon,
 & 
 \label{eq:simp-acc-constraint}\\
 &	\text{Expected completion time } \mathbb{E}[\tau] \le \theta,  
 \label{eq:simp-ect-constraint}
 \end{align}
 where $\epsilon$ and $\theta$ denote the maximum allowed error and the (wall-clock) job completion time respectively.

\subsection{Distributed SGD Primer}
\label{sec:sgd}

Most state-of-the-art machine learning systems employ Stochastic Gradient Descent (SGD) to train a neural network model so as to minimize the empirical risk function $G: \mathbb{R}^d \to \mathbb{R}$ over a training dataset $\mathcal{S}$, which is defined as 
\begin{align}\label{eq:avg_risk}
    G(\mathbf{w}) \triangleq \frac{1}{|\mathcal{S}|} \sum_{s=1}^{|\mathcal{S}|} l(h(x_s, \mathbf{w}), y_s),
\end{align}
where the vector $\mathbf{w}$ denotes the model parameters (for example, the weights and biases of a neural network model), and the loss $l(h(x_s, \mathbf{w}), y_s)$ compares our model's prediction $h(x_s, \mathbf{w})$ to the true output $y_s$, for each sample $(x_s, y_s)$. 

The mini-batch stochastic gradient descent (SGD) algorithm iteratively minimizes $G(\mathbf{w})$ by computing gradients of $l$ over a small, randomly chosen subset of data samples $\mathcal{S}_j$ in each iteration $j$ and updating $\mathbf{w}$ as per the update rule $\mathbf{w}_{j+1} = \mathbf{w}_j - \alpha_j g(\mathbf{w}_j)$. 
Here $\alpha_j$ is the (pre-specified) step size and  $g(\mathbf{w}_j) =\sum_{s\in \mathcal{S}_j} \nabla l(h(x_s, \mathbf{w}_j), y_s)/{|\mathcal{S}_j|}$, the gradient computed using samples in the mini-batch $\mathcal{S}_j$.

\textbf{Synchronous Distributed SGD.} To further speed up the training, many practical implementations parallelize gradient computation by using the parameter server framework shown in Fig.~\ref{fig:sgd_fig}. In this framework, there is a central parameter server and $n$ worker nodes. Each worker has access to a subset of the data, and in each iteration each worker fetches the current parameters $\mathbf{w}_j$ from the parameter server, computes the gradients of $l(h(x_s, \mathbf{w}_j), y_s)$ over one mini-batch of its data, and pushes them to the parameter server. The parameter server waits for gradients from all $n$ workers before updating the parameters to $\mathbf{w}_{j+1}$ as per 
\begin{align}
\mathbf{w}_{j+1} = \mathbf{w}_j - \frac{\alpha_j}{n} \sum_{i=1}^{n} g^{(i)}(\mathbf{w}_j), \label{eq:sync_sgd}
\end{align}
where $g^{(i)}(\mathbf{w}_j)$ is the mini-batch gradient returned by the $i^{th}$ worker. The updated $\mathbf{w}_{j+1}$ is then sent to all workers, and the process repeats. This gradient aggregation method is commonly referred to as synchronous SGD. Asynchronous gradient aggregation can reduce the delays in waiting for straggling workers, but causes staleness in the gradients returned by workers, which can give inferior SGD convergence \cite{Staleness:Aistats}. While we focus on synchronous SGD in this paper, the insights could be extended to other distributed SGD variants. 

\begin{figure}[t]
    \centering
    \includegraphics[width=8cm]{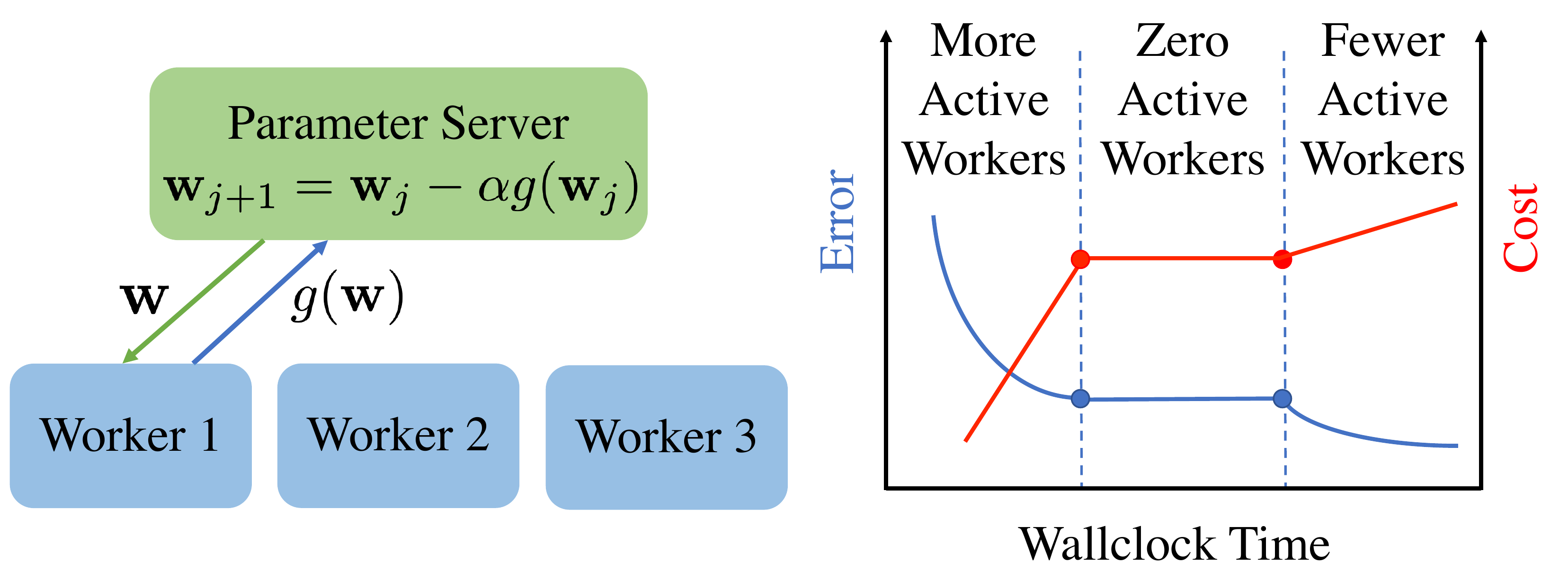}
    \caption{Parameter Server Model and an illustration of how error and cost vary versus training time when the number of workers varies with time. Having more active workers results in a faster decrease in error, but a faster increase in cost.}
    \label{fig:sgd_fig}
\end{figure}

\textbf{Distributed SGD on Volatile Workers.} In this work we consider that the parameter server is run on an on-demand instance, while the $n$ workers are run on volatile instances that can be interrupted or preempted during the training process, as illustrated in Fig.~\ref{fig:sgd_fig}. Let $y_j$ denote the number of active (i.e., not preempted) workers in iteration $j$, such that $0 < y_j \leq n$ for all $j = 1, \dots, J$, where $J$ is the total number of iterations. The sequence $y_1, y_2, \dots y_J$ can be considered as a random process. We do not count ``iterations'' where the number of active workers is $0$, as there is then no gradient update. However, having zero workers will increase the total training runtime, which we will account for in the runtime analysis in Section~\ref{sec:runtime_analysis}.

\subsection{SGD Error Convergence with Variable Number of Workers}
\label{sec:error_analysis}


Next we give an upper-bound on the expected training error in terms of $y_j$ for $j=1, \dots J$. For error convergence analysis we make the following assumptions on the objective function $G$, which are common in most prior works on SGD convergence analysis~\cite{Bottou,Staleness:Aistats}.

\begin{asmp}[Lipschitz-smoothness]\label{asmp-lipschitz}	
	The objective function $G(\mathbf{w}): \mathbb{R}^d \to \mathbb{R}$ is $L$-Lipschitz smooth, i.e., it is continuously differentiable and there exists $L > 0$ such that
	\begin{equation}
	\parallel  \nabla G(\mathbf{w}) - \nabla G(\mathbf{w}') \parallel_2 \le L \parallel  \mathbf{w} - \mathbf{w}' \parallel_2, \forall \mathbf{w}, \mathbf{w}' \in \mathbb{R}^d
	\end{equation}
\end{asmp}
\begin{asmp}[First and second moments]\label{asmp:moments}
	Let $\mathbb{E}_{\mathcal{S}_j} \left[ \nabla G(\mathbf{w}_j, \mathcal{S}_j)  \right]$ represent the expected gradient at iteration $j$ for a mini-batch $\mathcal{S}_j$ of the training data. Then there exist scalars $\mu_G \ge \mu > 0$ such that 
	\begin{align}
	\nabla G(\mathbf{w}_j)^T \mathbb{E}_{\mathcal{S}_j} \left[ \nabla G(\mathbf{w}_j, \mathcal{S}_j)  \right] &\ge  \mu \parallel  \nabla G(\mathbf{w}_j) \parallel_2^2\notag\\
	\text{and} \parallel \mathbb{E}_{\mathcal{S}_j} \left[ \nabla G(\mathbf{w}_j, \mathcal{S}_j)  \right]\parallel_2 
	&\le \mu_G \parallel  \nabla G(\mathbf{w}_j) \parallel_2
	\end{align}
	and scalars $M, M_V \ge 0$ and $M_G = M_V + \mu_G^2 \ge 0$ such that  
	\begin{align}
	\mathbb{E}_{\mathcal{S}_j} \left[ \parallel\nabla G(\mathbf{w}_j, \mathcal{S}_j)  \parallel_2^2 \right] \le M + M_G \parallel \nabla G(\mathbf{w}_j) \parallel_2^2.
	\end{align}
    for any given size of mini-batch $\mathcal{S}_j$ on one worker.
\end{asmp}



\begin{thm}[SGD Error Bound]
\label{thm:error-cv}
Suppose the objective function $G(\cdot)$ satisfies Assumptions \ref{asmp-lipschitz}-- \ref{asmp:moments} and is $c$-strongly convex \cite{Boyd:cv} with parameter $c\le L$. For a fixed step size $0< \alpha< \frac{\mu}{LM_G}$, the 
expected training error after $J$ iterations is:
\begin{align}\label{eq:error-cv}
    \mathbb{E}\left[ G(\mathbf{w}_{J}) - G^* \right]
    &\le (1-\alpha c \mu )^{J} \mathbb{E}\left[ G(\mathbf{w}_0)\right]+ \notag\\
    & \quad \frac{1}{2}\alpha^2 L M \sum_{j=1}^{J} (1-\alpha c \mu )^{J-j}\expect*{ \frac{1}{y_j}}
\end{align}
\end{thm}

The proof is given in the Appendix. The above convergence bound can be extended to handle non-convex objective function $G(\cdot)$ and a diminishing step size, where we analyze the convergence speed to a stationary point. We omit this extension for brevity.

\begin{remark}[Penalty for Using Volatile Instances]
\label{rem:penalty}
The error bound in Theorem~\ref{thm:error-cv} given the expected number of active workers $\mathbb{E}\left[y_j\right]$ is minimized when $y_j$ is not a random variable, i.e., SGD is run on on-demand instead of volatile instances. This result follows from the convexity of $y_j^{-1}$; using Jensen's inequality we can show that fixing the number of active workers to $y = \mathbb{E}\left[y_j\right]$ minimizes $\mathbb{E}\left[y_j^{-1}\right]$. 
\end{remark}

\begin{remark}[Error and Preemption Probability]
Suppose that a worker is preempted with probability $q$ in each iteration. Then the bound in Theorem~\ref{thm:error-cv} increases with $q$ because $\expect{1/y_j}$ increases with $q$. Thus, more frequent preemption or interruption of workers reduces the effective number of active workers and yields worse error convergence.   
\end{remark}

\subsection{SGD Runtime Analysis with Volatile Workers}
\label{sec:runtime_analysis}

Now let us analyze how using volatile workers affects the training runtime. The runtime has two components: 1) the time required to complete the $J$ SGD iterations, and 2) the idle time when no workers are active and thus no iterations can be run.

Let $R(y_j)$ denote the runtime of the $j^{th}$ iteration in which we have the set $\mathcal{Y}_j$ of $y_j$ active workers. Suppose each worker takes time $r_k$ to compute its gradient, where $r_k$ is a random variable. Fluctuations in computation time are common especially in cloud infrastructure due to background processes, node outages, network delays etc. \cite{stragglers-ls}. Since the parameter server has to wait for all $y_j$ workers to finish their gradient computations, the runtime per iteration is,
\begin{align}
R(y_j) = \max_{k \in \mathcal{Y}_j} r_k + \Delta,
\end{align}
where $\Delta$ is the time taken by the parameter server to update $\mathbf{w}$ and push it to the $y_j$ workers. The expected runtime $\expect*{R(y_j)}$ increases with the number of active workers. For example, if $r_k \sim \exp(\mu)$, an exponential random variable that is i.i.d.\ across workers and mini-batches, then $\expect*{R(y_j)} \approx (\log y_j)/\mu + \Delta$. Adding this per-iteration runtime to the idle time when no workers are active, we can show that the expected time required to complete $J$ SGD iterations is 
\begin{align*}
\label{constraint:completion_preemption}
\expect*{\tau} = 
\sum_{j=1}^J \mathbb{E}\left[R( y_j)\right] + \expect*{\text{idle time with no active workers}}
\end{align*}
For example, when each worker is preempted uniformly at random with probability $q$ in each iteration (as described in Remark 2), then the expected completion time becomes $ \expect*{\tau} = \sum_{j=1}^J \mathbb{E}\left[R( y_j)\right]/(1 - q^n)$.

\section{Optimizing Spot Instance Bids}\label{sec:spot-bid}

In this section, we use the results of Section~\ref{sec:error_runtime} to derive the bid prices and number of iterations that minimize the cost of running distributed SGD with workers placed on spot instances. We first consider the simple case in which we submit the same bid for each worker in Section~\ref{sec:uniform-bid} and then consider the heterogeneous bid case in Section~\ref{sec:DifferentBids}. 

\textbf{Spot Price and Bidding Model.} Let $p_t$ denote the spot price of each instance at time $t$. We assume $p_t$ is i.i.d. and is bounded between a lower-bound $\ubar{p}$ and an upper-bound $\bar{p}$, similar to prior works on optimal bidding in spot markets~\cite{BidCloud}. Let $f(\cdot)$ and $F(\cdot)$ denote the probability density function (PDF)~\cite{wiki:pdf} and the cumulative density function (CDF)~\cite{wiki:cdf} of the random variable $p_t$. 
When a bid $b$ is placed for an instance, we consider that the provider assigns available spot capacity to users in descending order of their bids, stopping at users with bids below the prevailing spot price. Thus, a worker is active only if its bid price exceeds the current spot price. Hence, without loss of generality the range of the bid price can also be assumed to be $\ubar{p} \le b \le \bar{p}$. Whenever a worker is active ($b \geq p_t$), the per-time cost incurred for running it is equal to the prevailing spot price $p_t$ (not the bid price).


\subsection{Identical Worker Bids}
\label{sec:uniform-bid}
Suppose we choose bid price $b$ for each of the $n$ workers. We first simplify the error and runtime in  Section~\ref{sec:error_runtime} for this case, and then solve the cost minimization problem (\ref{eq:simp-obj-cost})-(\ref{eq:simp-ect-constraint}). 

%
%
Observe that the $n$ workers are either all available or all interrupted depending on the bid price $b$.
%
%
This insight implies that $\mathbb{E}\left[y_j^{-1}\right] = 1/n$, and thus that the error bound in Theorem~\ref{thm:error-cv} is \emph{independent of the bid $b$}: this bid affects only the frequency with which iterations are executed, not the number of active workers in an iteration. We can thus rewrite the error bound as a function of $J$, the number of iterations required to reach error $\epsilon$. 
Formally, we set $\hat{\phi}$ to be the right-hand side of \eqref{eq:error-cv} and $J \geq \hat{\phi}^{-1}(\epsilon)$, where $\hat{\phi}^{-1}(\epsilon)$ is the number of iterations required to ensure that the expected error is no larger than $\epsilon$.

We further observe that, the number of active workers $y_j$ always equals $n$ when the job is running. Thus, the expected runtime per iteration can be rewritten as $\mathbb{E}\left[R(y_j   )\right] = \expect*{R(n)}$. 
Accounting for the idle time we can show that the expected completion time is monotonic with $b$:

\begin{lem}[Completion Time in Terms of Bid Price]\label{lem:single-bid-completion-monotonicity}
Using the same bid price $b$ for all workers, the expected completion time to complete $J$ iterations of synchronous SGD is
 \begin{equation}\label{eq:expected-completion-time}
 \expect*{\tau} = J \expect*{R(n)} / F(b), 
 \end{equation}
which increases with $J$ and is non-increasing in the bid price $b$. The function $F(\cdot)$ is the CDF of the spot price.
\end{lem}

We can further show the expected cost (defined in \eqref{eq:simp-obj-cost}) is monotonically non-decreasing with $b$ and $J$.
\begin{lem}[Cost in Terms of Bid Price]\label{lem:single-bid-cost-monotonicity}
Using one bid price for all workers, the expected cost of finishing a synchronous SGD job is given by
\begin{equation}\label{eq:cost-simp}
     \expect*{C} = J n \expect*{R(n)} \left( \ubar{p} + \int_{\ubar{p}}^{b}  \left(1-\frac{F(p)}{F(b)}\right)\text{d}p \right),
\end{equation}
which is non-decreasing in the bid price $b$ and $J$. The function $F(\cdot)$ is the CDF of the spot price.
\end{lem}

Since both $\expect*{\tau}$ and $\expect*{C}$ increase with $J$, we should set $J$ to be equal to $\hat{\phi}^{-1}(\epsilon)$ in order to reach the target error in minimum time and cost of the volatile workers.


{\bf Optimizing the Bid Price.} Having shown that $J = \hat{\phi}^{-1}(\epsilon)$, we now find the optimal bid $b$ that minimizes the expected cost \eqref{eq:cost-simp} to solve the optimization problem~\eqref{eq:simp-obj-cost}--\eqref{eq:simp-ect-constraint}.  %

According to Amazon's policy~\cite{spot}, $b$ is determined upon the job submission without knowing the future spot prices and will be fixed for the job's lifetime. Although the user can effectively change the bid price by terminating the original request and re-bidding for a new VM, doing so induces significant migration overhead. Thus, we assume that users employ persistent spot requests: a worker with a persistent request will be resumed once the spot price falls below its bid price, exiting the system once its job completes.
%
%
Using Lemma~\ref{lem:single-bid-completion-monotonicity} and Lemma~\ref{lem:single-bid-cost-monotonicity}, we can show the following theorem for the optimal bid price $b$. 

\begin{thm}[Optimal Uniform Bid]\label{thm:identical-opt-bids}
When we make an identical bid $b$ for $n$ workers and use them to perform distributed synchronous SGD to reach error $\epsilon$ within time $\theta$, the optimal bid price that minimizes the cost is 
%
$b^* = F^{-1}\left( \frac{\hat{\phi}^{-1}(\epsilon)\expect*{R(n)}}{\theta} \right)$. 
\end{thm}
Theorem \ref{thm:identical-opt-bids} provides a general form of the optimal bid price, given the number of workers per iteration, $n$, the deadline $\theta$, and the target error bound $\epsilon$, for any distributions of the spot price and training runtime per iteration. 

\subsection{Optimal Heterogeneous Bids}
\label{sec:DifferentBids}

We next extend our results from Section~\ref{sec:uniform-bid} to find the optimal bidding strategy with two distinct bid prices $b_1$ and $b_2$ for two groups of workers. This strategy is motivated by the observation that 
bidding lower prices for some workers yields a larger number of active workers when the spot price is relatively low, which improves the training error but will not cost much.
 %
%
Formally, we place bids of $b_1$ for workers $1, \cdots, n_1$ and $b_2$ ($< b_1$) for workers $n_1+1, \cdots, n$. We define the random variable $y( \vec{b}) \in \{n_1, n\}$ as the number of active workers when the bid prices are $\vec{b} = (b_1, b_2)$. Note that the times when $0$ workers are active are not considered into an SGD `iteration'. Thus, $y( \vec{b})$ can only be either $n_1$ (with probability $\frac{F(b_1) - F(b_2)}{F(b_1)}$) or $n$ (with probability $F(b_2)/F(b_1)$) in each iteration. 

{\bf Optimized Bids.} We initially assume that $n_1$, the number of workers in the first group, and $J$, the required number of iterations, are fixed; thus, we optimize the trade-off between the expected cost, expected completion time, and the expected training error using only the bid prices $\vec{b}$.
After deriving the closed-form optimal solutions of $b_1$ and $b_2$ in Theorem \ref{thm:opt-two-bids-fully-sync}, we discuss co-optimizing $n_1$ and $J$ with the bids $\vec{b}$.
%
The expected cost minimization problem ~\eqref{eq:simp-obj-cost}--\eqref{eq:simp-ect-constraint} then becomes:
\begin{align}
    \min_{\vec{b}} \quad 
    & J \int_{\ubar{p}}^{b_1} \expect*{R(\vec{b}, p)} y(\vec{b}) p\frac{f(p)}{F(b_1)} \text{d}p\label{obj:fully-sync}\\
    \text{subject to:}\quad  
    & \expect*{\hat{\phi}(\vec{b})} \le \epsilon
    \quad \text{(Error constraint)}
    \label{constraint:error-fully-sync}\\
    & \frac{J}{F(b_1)} \int_{\ubar{p}}^{b_1} \expect*{R(\vec{b}, p)}\frac{f(p)}{F(b_1)}\text{d}p 
    \le \theta \label{constraint:ddl-fully-sync}\\
    & \bar{p} \ge b_1 \ge b_2 \ge \ubar{p},~\forall i\le j \label{constraint:bids-fully-sync}
\end{align}
To derive the cost and completion time expressions in~(\ref{obj:fully-sync}) and (\ref{constraint:ddl-fully-sync}), 
we express the expected runtime of iteration $j$ as $\expect*{R(\vec{b}, p)}$, a function of the bids and price; $y_j$ depends on $\vec{b}$ and thus is re-written as $y(\vec{b})$.
For simplicity, we assume that the spot prices do not change within each iteration. 
In practice, the spot price changes at most once per hour 
\cite{price_change}, compared to a runtime of several minutes per iteration, and thus this assumption usually holds. Note that we did not need this assumption for the identical bid case in Section~\ref{sec:uniform-bid} since all workers become active/inactive at the same time. 


To derive the optimal bid prices, we first relate the distribution of the spot price and our bid prices to the training error through the number of active workers, \ie, $y(\vec{b})$. From Theorem \ref{thm:error-cv}, the expected error is at most $\epsilon$ if $y(\vec{b})$ satisfies: 
\begin{align}\label{eq:Qe}
    \expect*{\frac{1}{y(\vec{b})}} 
    \le  
    \frac{2c \mu \left( \epsilon - (1-\alpha c \mu)^J  \expect*{G(\mathbf{w}_0)}  \right)}{\alpha L M \left( 1 - (\alpha c \mu)^{J} \right)} \triangleq Q(\epsilon)
\end{align}
Further, we simplify $\expect*{R(\vec{b},p)}$ to be a function of the number of active workers: $\expect*{R(X)}$ is the expected runtime per iteration given $X$ workers are active. We then provide closed-form expressions for the optimal bid prices through Theorem \ref{thm:opt-two-bids-fully-sync}.

\begin{thm}[Optimal-Two Bids with a Fixed $J$]
\label{thm:opt-two-bids-fully-sync}
Suppose the objective function $G(\cdot)$ satisfies Assumptions \ref{asmp-lipschitz}--\ref{asmp:moments}. 
Given a number of iterations ($J$) that can guarantee $\frac{1}{n} < Q(\epsilon) \le \frac{1}{n_1}$ ($Q(\epsilon)$ is defined as the right-hand side of \eqref{eq:Qe}), a fixed step size $\alpha$, and a feasible deadline ($\theta \ge J\expect*{R(n)}$), we have the optimal bid prices $b^{\ast}_1$ and $b^{\ast}_2$:
\begin{align}
    b^{\ast}_1 &= F^{-1}\left(\frac{J}{\theta} \left( \left( \expect*{R(n)} - \expect*{R(n_1)} \right)\frac{\frac{1}{n_1} - Q(\epsilon)}{\frac{1}{n_1} - \frac{1}{n}} + \expect*{R(n_1)} \right) \right)\notag\\
    b^{\ast}_2 &= F^{-1}\left(\frac{\frac{1}{n_1} - Q(\epsilon)}{\frac{1}{n_1} - \frac{1}{n}} \times F(b^{\ast}_1)\right)\label{eq:opt-two-bids-fully-sync},
\end{align}
for any i.i.d. spot price and any i.i.d. running time per mini-batch, \ie, $F(\cdot)$ and $\expect*{R(n)}$ (or $\expect*{R(n_1)}$) do not change during the training process.
\end{thm}

\captionsetup[figure]{labelfont=bf}
\begin{figure*}[t]
    \centering
    \begin{subfigure}[t]{0.19\linewidth}
    \includegraphics[width=\textwidth]{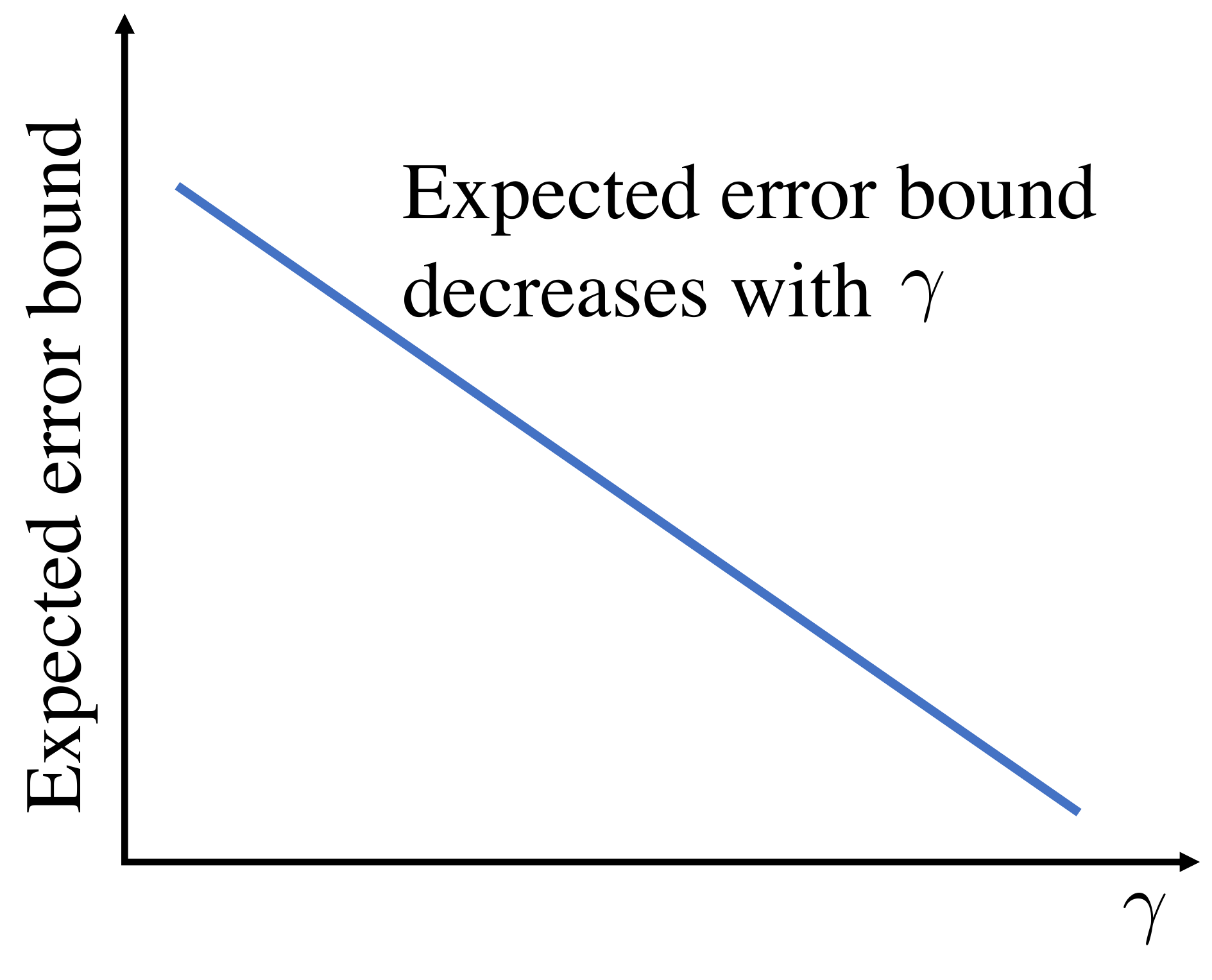}
    \caption{Error-vs-$\gamma$}
    \label{fig:illustration-error-gamma}
    \end{subfigure}
    \begin{subfigure}[t]{0.19\linewidth}
    \includegraphics[width=\textwidth]{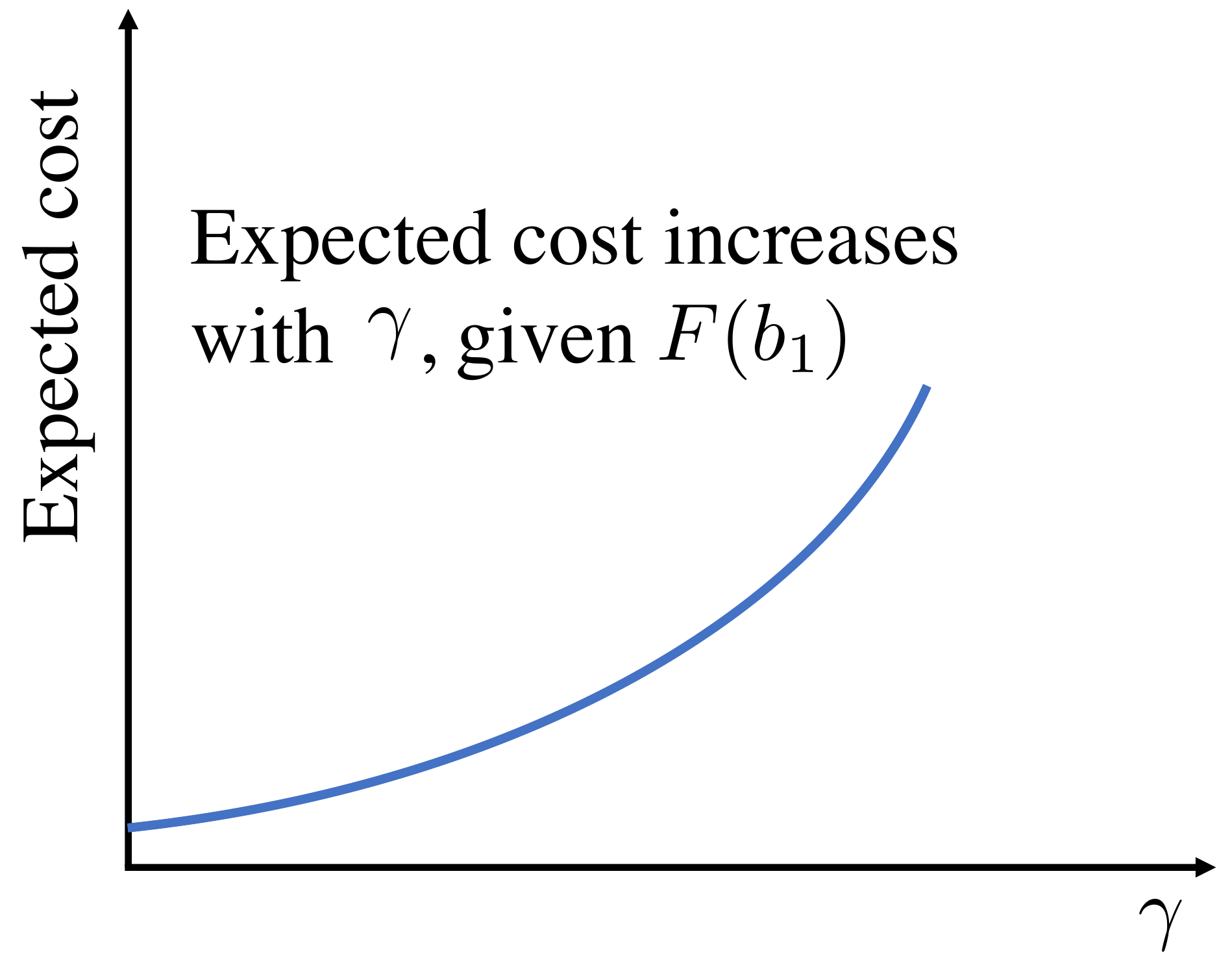}
    \caption{Cost-vs-$\gamma$}
    \label{fig:illustration-cost-gamma}
    \end{subfigure}
    \begin{subfigure}[t]{0.19\linewidth}
    \includegraphics[width=\textwidth]{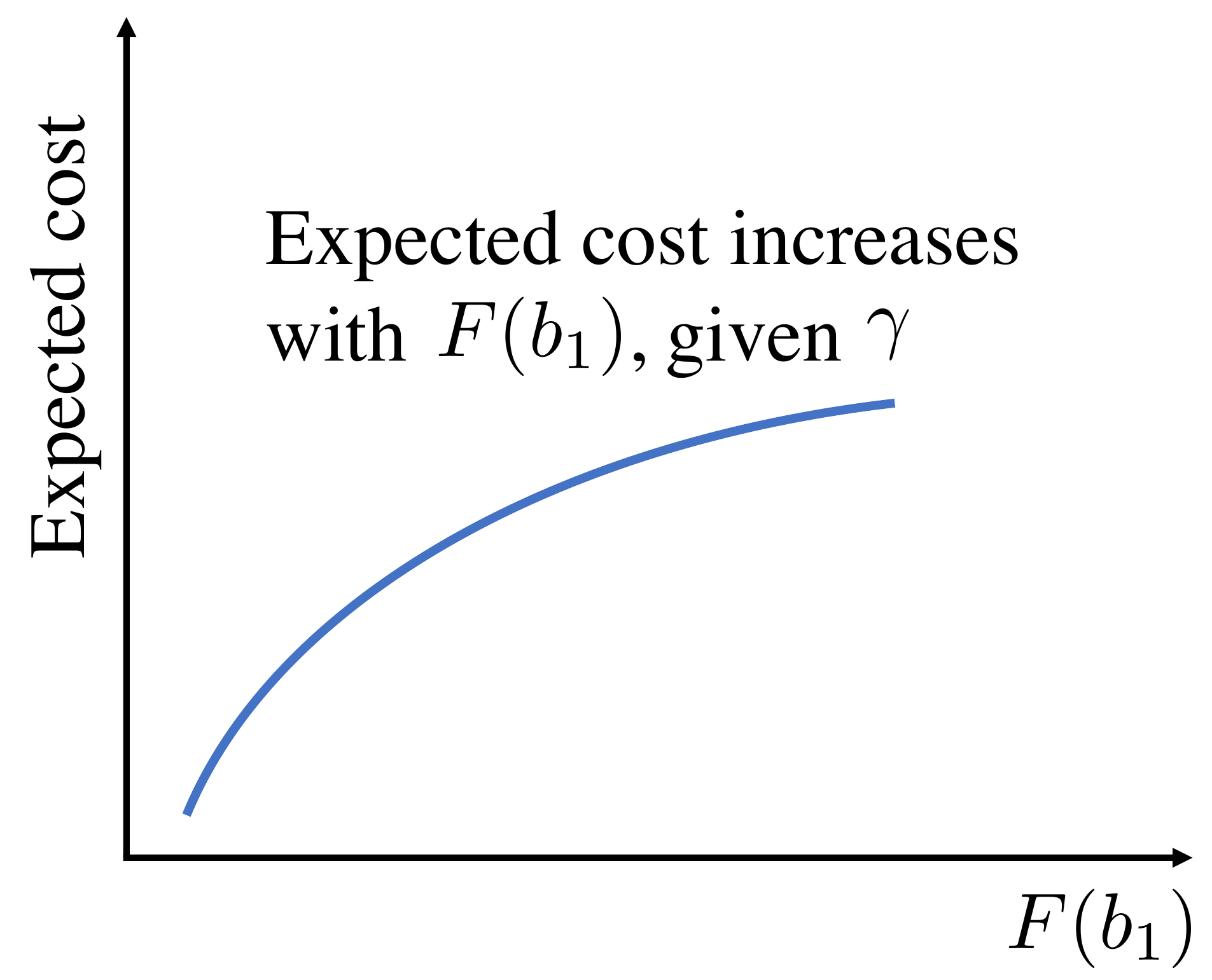}
    \caption{Cost-vs-$F(b_1)$}
    \label{fig:illustration-cost-F-b1}
    \end{subfigure}
    \begin{subfigure}[t]{0.19\linewidth}
    \includegraphics[width=\textwidth]{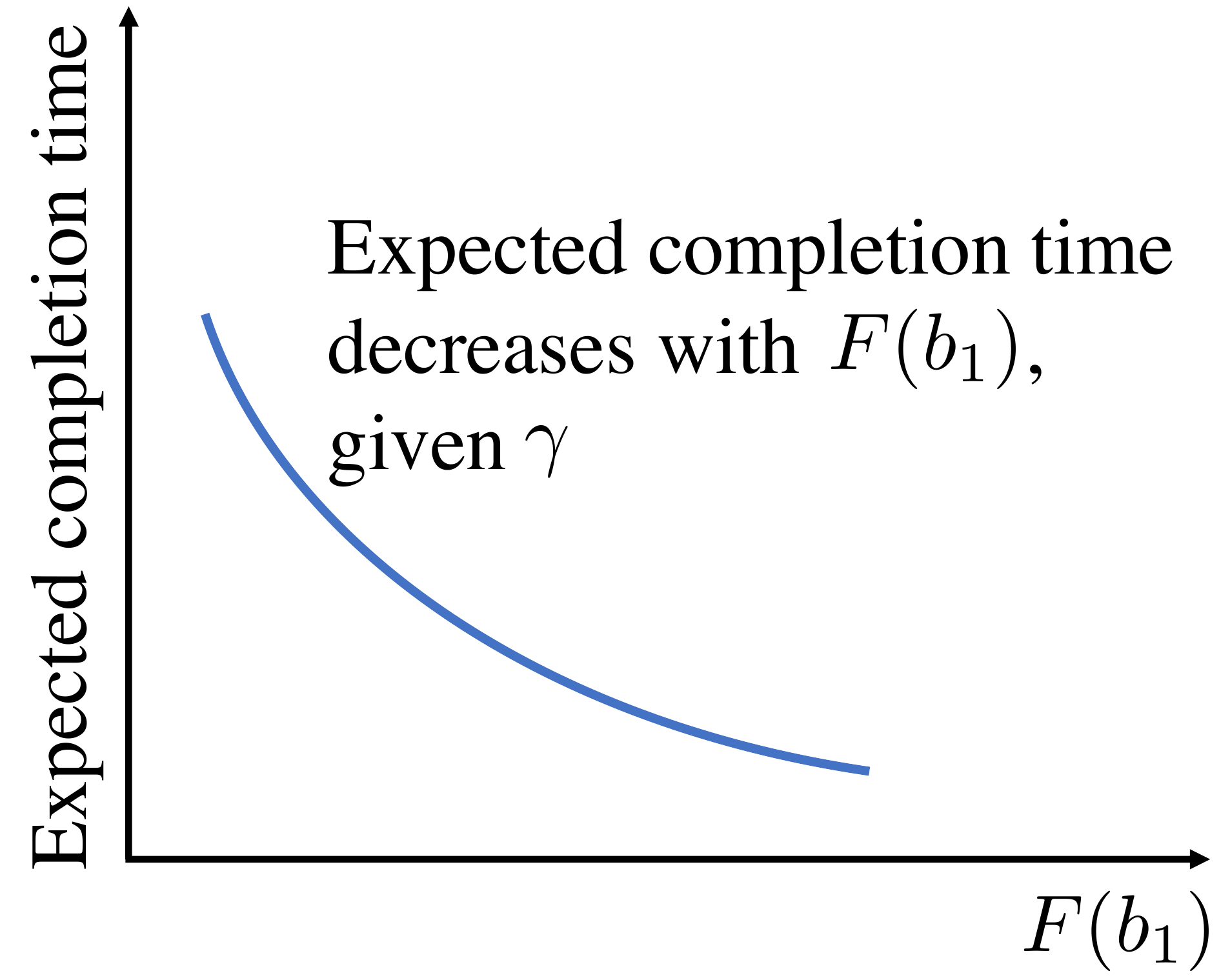}
    \caption{Compl.~time-vs-$F(b_1)$}
    \label{fig:illustration-time-F-b1}
    \end{subfigure}
    \begin{subfigure}[t]{0.19\linewidth}
    \includegraphics[width=\textwidth]{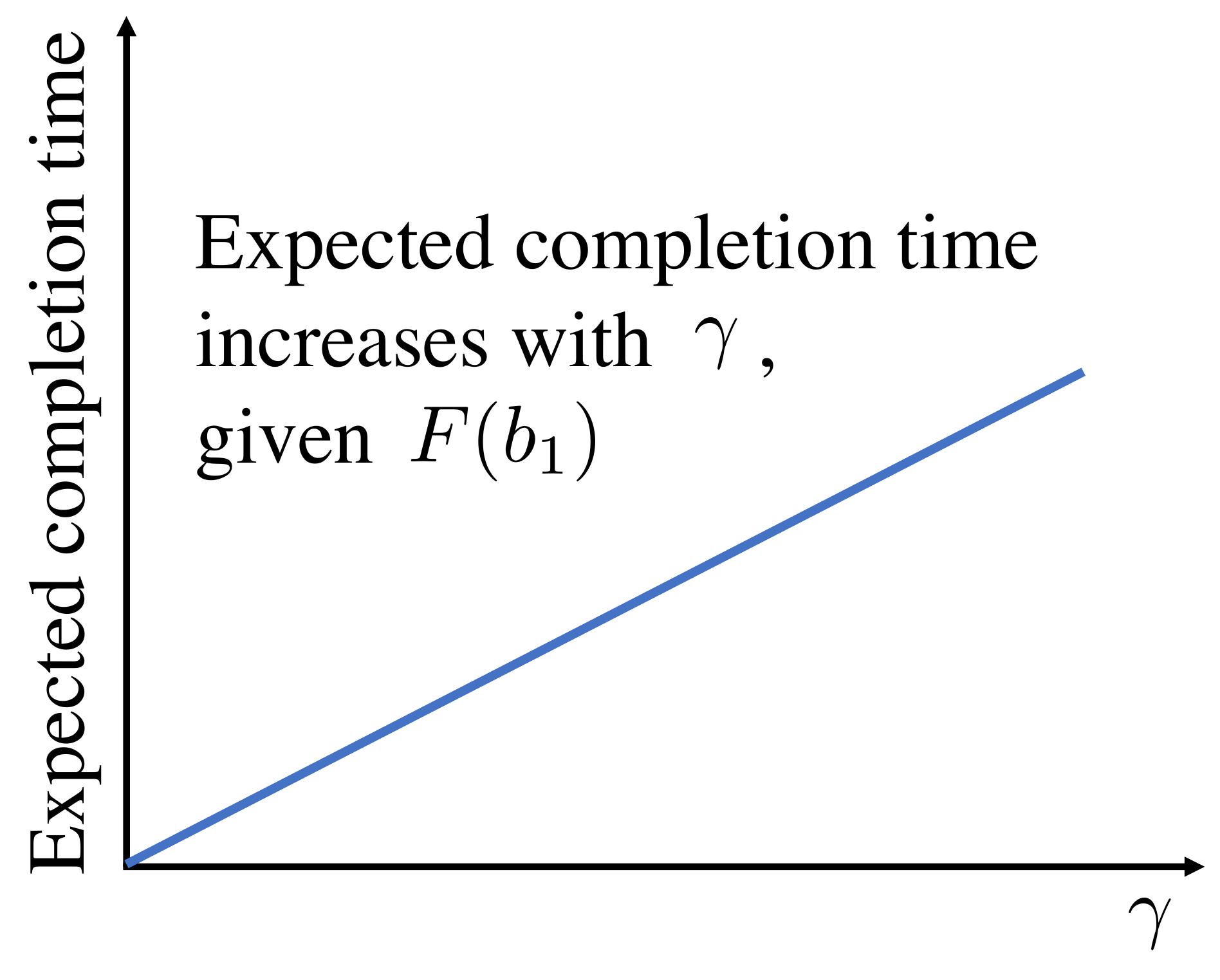}
    \caption{Compl.~time-vs-$\gamma$}
    \label{fig:illustration-time-gamma}
    \end{subfigure}
    \caption{
    Illustration of how the expected cost, completion time and error vary w.r.t. $F(b_1)$ and $\gamma = \frac{F(b_2)}{F(b_1)}$. As a larger $\gamma$ leads to a smaller expected error (Fig.~\ref{fig:illustration-error-gamma}) but a larger expected cost (Fig.~\ref{fig:illustration-cost-gamma}) and completion time (Fig.~\ref{fig:illustration-time-gamma}), and the expected error is only controlled by $\gamma$, the optimal $\gamma$ should be the smallest possible $\gamma$, \ie, the one that yields error $= \epsilon$. The optimal $F(b_1)$ should be the one that yields the completion time equal to the deadline under the optimal $\gamma$ (Fig.~\ref{fig:illustration-time-F-b1}).}
    \label{fig:illustration-optimal-two-bids}
\end{figure*}
For brevity, we use Figure \ref{fig:illustration-optimal-two-bids} to illustrate our proof of Theorem \ref{thm:opt-two-bids-fully-sync}. The key steps are: (i) change the variables of the optimization problem \eqref{obj:fully-sync} to be $F(b_1)$ and $\gamma = \frac{F(b_2)}{F(b_1)}$; (ii) show that the expected cost, completion time, and error are monotonic w.r.t. to $F(b_1)$ and $\gamma$. Intuitively, the expected error should depend only on the number of active workers \emph{given that some workers are active}, which is controlled by the relative difference between $F(b_1)$ and $F(b_2)$: $\gamma$. Formally, the error bound decreases with $\expect*{y(\vec{b})^{-1}}$. Applying $\expect*{y(\vec{b})^{-1}} = \frac{1}{F(b_1)}\left(\frac{F(b_1)-F(b_2)}{n_1} + \frac{F(b_2)}{n}\right)= \frac{1}{n_1} - \frac{1}{\gamma} \left(\frac{1}{n_1} - \frac{1}{n}\right)$ to \eqref{eq:Qe} gives us the optimal $\gamma$, since the expected cost increases with both $F(b_1)$ and $\gamma$. We then choose $F(b^*_1)$ to the one that yields $\expect*{\tau} = \theta$ (tight \eqref{constraint:ddl-fully-sync}). Intuitively, $F(b^*_1)$ should be high enough to guarantee that some workers are active often enough that the job completes before the deadline. 


{\bf Co-optimizing $n_1$ and $\vec{b}$.} 
If $n_1$ is not a known input but a variable to be co-optimized with $\vec{b}$, we can write $n_1$ and $b^*_2$ in terms of $F(b^*_1)$ according to \eqref{eq:opt-two-bids-fully-sync} and plug them into \eqref{obj:fully-sync}-\eqref{constraint:bids-fully-sync} to solve for $b^*_1$ first, and then derive $b_2^*$ and the optimal $n_1$. 

{\bf Co-optimizing $J$ and $\vec{b}$.} Taking $J$ as an optimization variable may allow us to further reduce the job's cost. For instance, allowing the job to run for more iterations, \ie, increasing $J$, increases $Q(\epsilon)$ (the right-hand side of ~\eqref{eq:Qe}). We can then increase $\expect*{\frac{1}{y(\vec{b})} }$ by submitting lower bids $b_2$, making it less likely that workers $n_1 + 1,\ldots,n$ will be active, while still satisfying \eqref{eq:Qe}. A lower $b_2$ may decrease the expected cost by making workers less expensive, though this may be offset by the increased number of iterations. To co-optimize $J$, we show it is a function of $\vec{b}$ and $\epsilon$: 
\begin{cor}[Relationship of $J$ and $\vec{b}$]\label{lem:J-b-epsilon}
To guarantee a training error $\leq\epsilon$, the number of iterations $J$ should be at least
\begin{align}\label{eq:J-b-epsilon}
    &J = \log_{(1-\alpha c\mu)} \frac{\epsilon - \frac{\alpha LM}{2c\mu} \expect*{\frac{1}{y(\vec{b})}}  }{\expect*{G(\mathbf{w}_0)} - \frac{\alpha LM}{2c\mu} \expect*{\frac{1}{y(\vec{b})}}}.
\end{align}
\end{cor}
For brevity, we show the idea of co-optimizing $J$ and $\vec{b}$:
We first replace $J$ in \eqref{obj:fully-sync} and \eqref{constraint:ddl-fully-sync} by \eqref{eq:J-b-epsilon}. Constraint \eqref{constraint:error-fully-sync} is already guaranteed by \eqref{eq:J-b-epsilon} and can be removed. We then solve for the remaining optimization variables, the bids $\vec{b}$. 

\section{Optimal Number of Preemptible Instances}
\label{sec:preempt}

In this section, we consider preemptible instances offered by other cloud platforms, \eg, low priority VMs from Microsoft Azure~\cite{low_priority_azure} and preemptible instances from Google Cloud Platform~\cite{gcp}. Unlike spot instances where users can specify the maximum prices they are willing to pay, on these platforms users can only decide the number of provisioned instances to request in each iteration, as well as the number of iterations. Therefore, in this section, we choose to optimize the number of instances (workers) and assume the instance price is stable during the entire training time~\cite{gcp}. 
To better quantify the relationship between the number of active workers $y_j$ and the number of provisioned workers $n$, we consider the two preemption distributions in Lemma~\ref{lem:distributions_preemption}.
We will make use of the fact that for both distributions, there exists a parameter $\chi>0$ such that $\expect*{\frac{1}{y_j}} \leq O\left(\frac{1}{n^{\chi}}\right)$.
The problem of minimizing the job cost is then equivalent to minimizing $\mathbb{E}\left[\sum_{j=1}^J y_j R\left(y_j\right) \right]$, subject to the completion time and error constraints. 

\begin{lem}[Example Distributions of $y_j$]
\label{lem:distributions_preemption}
If the number of active workers $y_j$ follows a uniform distribution $\mathbb{P}[y_j = k] = \frac{1}{n_j}, \forall k = 1, \cdots, n_j$, we have $\expect*{\frac{1}{y_j}} \le O\left(n_j^{\frac{-1}{2}}\right)$; if each worker is preempted with probability $q$ each iteration, we have $\expect*{\frac{1}{y_j}} \le O\left(\frac{1}{n_j^{\chi}}\right)$, where there exists a $\chi \in (0,1)$.
\end{lem}

We find closed-form solutions for the optimal number of workers $n$ and iterations $J$ when $\chi \ge 1$ in Theorem \ref{thm:co-opt_n_J}. Theorem \ref{thm:dynamic_nj} provides an optimization strategy for any $\chi>0$.
\begin{thm}[Co-optimizing $n$ and $J$]
\label{thm:co-opt_n_J}
Suppose $\expect*{y_j} \propto n$ and $\expect*{\frac{1}{y_j} } \le \frac{d}{n}~(d>0)$, the probability of no active workers does not depend on $n$, and the runtime per iteration is deterministic. Then the completion time constraint \eqref{eq:simp-ect-constraint} is simply $J\le \theta \delta$ where $\delta$ is a constant, and the optimal $J$ and $n$ (denoted by $J^*$ and $n^*$) satisfy: 
\begin{align*}
    J^* &= \min\left\{\argmin_{J \in \{J_1, J_2\}} \frac{BJ(1 -\beta^{J})}{(1-\beta)(\epsilon - A\beta^{J})}, \floor*{\theta \delta}\right\},\\
    J_1 &= \floor*{\tilde{J}}, 
    ~ J_2 = \ceil*{\tilde{J}},
    \frac{A\beta^{\tilde{J}}\left( \tilde{J}\ln{\frac{1}{\beta}} + 1 - \beta^{\tilde{J}} \right)}{1 + \beta^{\tilde{J}}(\tilde{J}\ln{\frac{1}{\beta}} - 1)} = \epsilon,\\
    n^* &= \ceil*{ \frac{B(1 -\beta^{\tilde{J}})}{(1-\beta)(\epsilon - A\beta^{\tilde{J}})}},
\end{align*}
where $\beta = 1 - \alpha c \mu$, $A = \expect*{G(\mathbf{w}_0)}$, and $B=\frac{\alpha^2 L M d}{2}$.
\end{thm}

{\bf A Strategy with Dynamic Numbers of Workers.}~While Theorem~\ref{thm:co-opt_n_J} gives us the exact optimal expression for $n$ when the provisioned number of workers is fixed over iterations, ML practitioners often increase the number of workers over time~\cite{icml_increase_nodes, vldb_increase_machines, iclr_backup_workers}. 
Intuitively, in the later stages of the model training the parameter values are closer to convergence, and thus it is crucial that the gradient updates are accurate, i.e., averaged over a larger number of worker mini-batches. More formally, we observe in Theorem \ref{thm:error-cv} that $\expect*{\frac{1}{y_j}}$'s contribution to the error bound increases exponentially with $j$ by $\frac{1}{1-\alpha c \mu}$. 

Inspired by these observations, we propose to decrease $\expect*{\frac{1}{y_j}}$ over iterations by controlling the provisioned number of workers: we dynamically set the number of workers to be $n_j = \ceil*{n_0 \eta^{j-1}}$ for each iteration $j$ and some $\eta > 1$; we show how to optimize the value of $\eta$ below. One can similarly exponentially increase the batch size of each worker while using the same number of workers over iterations~\cite{icml_exp_batch}, 
but doing so will exponentially increase the runtime of each iteration. We prove in Theorem \ref{thm:dynamic_nj} that our dynamic strategy achieves the same error convergence rate and a better asymptotic error bound with a significantly smaller number of iterations than using a static number of workers during the entire training.

\begin{thm}[Error with Dynamic Workers]
\label{thm:dynamic_nj}
Suppose the number of active workers $y_j$ satisfies $\expect*{\frac{1}{y_j}} \le O\left(\frac{1}{n_j^{\chi}}\right)$ for some $\chi \ge 0$. 
Then for any $\eta > 1$ and $J$ sufficiently large, provisioning $\ceil*{n_0 \eta^{j-1}}$ workers in iteration $j$ and running  
SGD for $\ceil*{\log_{\eta^{\chi}} \left( 1+ (\eta -1)J\right)}$ iterations achieves an error bound no larger than provisioning $n_0$ workers for $J$ iterations.
\end{thm}
In the proof of Theorem~\ref{thm:dynamic_nj}, we also show that our dynamic strategy achieves an error bound that converges to $0$ asymptotically with $J$, while when using a static number of workers the error bound in Theorem~\ref{thm:error-cv} converges to a positive constant. 

We then optimize $\eta$ to minimize the expected cost, subject to the error and completion time constraints.
If we ignore straggler effects, we can define $\mathbb{E}\left[R(y_j)\right] = R, ~\forall j$. Suppose $z_j$ denotes the number of active workers including the case $z_j=0$, and $z_j$ follows a binomial distribution with parameter $n_j$ and probability $q$ (the probability that each instance is inactive), namely, the probability that $z_j=0$ equals $q^{n_0 \eta^j}$. Assuming $\expect*{y_j} \propto n_j = n_0 \eta^{j-1}$ and $\expect*{\frac{1}{y_j}} \le \frac{d}{n_j^{\chi}}$, our cost minimization problem can be modified as follows.
\begin{align}
    \rm{minimize}_{\eta} \quad
    & (1-\eta^J)/(1-\eta) \label{eq:cost_preempt_2}\\
    \rm{subject~to:} \quad
    & \sum_{j=1}^J R/(1-q^{n_0 \eta^j}) \le \theta  \label{constraint:completion_preempt_2}\\
    & A\beta^{J} + \frac{B\beta^{J-1}\left(1-(\frac{1}{\beta \eta^{\chi} })^J\right)}{n_0^{\chi}\left(1- \frac{1}{\beta \eta^{\chi} }\right)} \le \epsilon \label{eq:constraint_error_2}\\
    & \eta^{\chi} > 1 / \beta \label{eq:constraint_others_preempt_2},
\end{align}
where $\beta = 1-\alpha c \mu $, $A = \mathbb{E}\left[ G(\mathbf{w}_0) \right]$, and $B = \frac{\alpha^2 LM d}{2}$. For any given $J$, both the objective function and constraints are convex functions of $\eta$ (refer to the operations that preserve convexity in \cite{Boyd:cv}). Therefore, we can use standard algorithms for convex optimization to solve for the optimal $\eta$. 

We can capture the effect of straggling workers by replacing the constant per-iteration runtime $R$ in \eqref{eq:simp-ect-constraint} with $\mathbb{E}\left[R(y_j)\right] = \frac{1}{\lambda}\left( \log n_0 + (j-1)\log \eta\right)$ in the completion time constraint \eqref{eq:simp-ect-constraint}. This constraint accounts for the fact that as we have more active workers in each iteration, the per-iteration runtime will likely increase because we need to wait for the slowest worker to finish.
As in the case without stragglers, we then observe that our optimization problem is convex in $\eta$ for each fixed $J$, and moreover that there exists a finite maximum number of iterations $J$ for which (\ref{eq:simp-ect-constraint}) is feasible. Thus, we can jointly optimize the optimal rate of increase in the number of workers, $\eta$, and $J$ by iterating over all possible values of $J$.
\section{Experimental Validation}
\label{sec:experiments}
We evaluate our bidding strategies from Section~\ref{sec:spot-bid} on the CIFAR-10 image classification benchmark dataset, using $J = 5000$ iterations on ResNet-50 \cite{resnet} and $J=10000$ on a small Convolutional Neural Network (CNN) \cite{cnn-nips2012} with two convolutional layers and three fully connected layers; the distributed SGD algorithms under both datasets are implemented based on Ray~\cite{ray} and Tensorflow~\cite{Tensorflow}. We run the former experiments on a local cluster with GPU servers and the latter on Amazon EC2's c5.xlarge spot instances. 


{\bf Choosing the Experiment Parameters.} We set the deadline ($\theta$) to be twice the estimated runtime of using $8$ workers to process $J$ iterations without interruptions. 
We estimate that $Q(\epsilon) \in [\frac{1}{n}, \frac{1}{n_1}]$ for our choices of $\epsilon$ and $J$ ($\epsilon = 0.98$ for ResNet-50 and $\epsilon = 0.65$ for the small CNN), demonstrating the robustness of our optimized strategies to mis-estimations. 
To estimate the probability distribution of the spot prices, 
we first consider two synthetic spot price distributions for the ResNet-50 experiments: a uniform distribution in the range $[0.2, 1]$ and a Gaussian distribution with mean and variance equal to $0.6$ and $0.175$; we draw the spot price when each iteration starts and re-draw it every $4$ seconds after the job is interrupted. We then download the historical price traces of c5.xlarge spot instances using Amazon EC2's DescribeSpotPriceHistory API for the small CNN experiments, demonstrating that our bidding strategy is robust to non-i.i.d~spot prices. 


{\bf Superiority of our Bidding Strategies.} 
We evaluate the bidding strategies with both the optimal single bid price for all workers ({\bf Optimal-one-bid}) and the optimal bid prices for two groups of workers derived in Theorem \ref{thm:opt-two-bids-fully-sync} ({\bf Optimal-two-bids}) against an aggressive {\bf No-interruptions} strategy that chooses a bid price larger than the maximum spot price. To further minimize the expected total cost while guaranteeing a low training/test error, we propose a {\bf Dynamic strategy}, which updates the optimal two bid prices when increasing the total number of workers. More specifically, we initially launch four workers ($n_1 = 2$, $n = 4$) and apply our optimal two bid prices. After completing 4000 iterations, we add four more workers ($n_1 = 4,n = 8$) and re-compute the optimal bids by subtracting the consumed time from the original deadline $\theta$ and taking $J$ to be the number of remaining iterations. One could further divide the training and re-optimization into more stages. Frequent re-optimizing will likely incur significant interruption overheads, but infrequent optimization may reduce the cost with tolerable overhead.

Figures~\ref{fig:cif} and~\ref{fig:data_cif} compare the performance of our strategies on synthetic and real spot prices, respectively. Figures~\ref{fig:cif-acc-cost-uniform} and~\ref{fig:cif-acc-cost-normal} show 
that \emph{our dynamic strategy leads to a lower cost and the no interruptions benchmark to a higher cost for any given accuracy}, compared to the optimal-one-bid and optimal-two-bids strategies. 
In Figures~\ref{fig:cif-cost-uniform} and~\ref{fig:cif-cost-normal}, we 
indicate the cumulative cost as we run the jobs. The  markers indicate the costs where we achieve 98\% accuracy; while the no interruptions benchmark achieves this accuracy much faster, it \emph{costs nearly three times as much as our dynamic strategy and twice as much as our optimal-two-bids strategy}. Figures \ref{fig:data_cif_acc_cost} and \ref{fig:data_cif_cost_time} show that our optimal-one-bid and optimal-two-bids strategies can significantly save cost under the real spot prices while achieving almost the same training accuracy as the no interruptions benchmark. 



{\bf Superiority of Our Choices of the Number of Workers.} To verify our results in Section \ref{sec:preempt}, we simulate {\bf No preemption} by running 2 workers for 10000 iterations without preemption and observe that the final accuracy can approach $63\%$. 
We then suppose instances are preempted with probability $p = 0.5$ and provision $n = 4$ workers for $J=10000$ iterations, using the fact that the optimal $n$ for each fixed $J$ is proportional to $1/(1 - p)$ and aiming to achieve the same accuracy $65\%$. Co-optimizing $n$ and $J$ (Theorem~\ref{thm:co-opt_n_J}) may yield further cost improvements.
Figure \ref{fig:preemption_distributions} shows that using our estimated $n$ achieves a better accuracy per dollar than randomly choosing $n$. We further show in Figure \ref{fig:preemption_strategies} that our strategy {\bf Dynamic $n_j$}, which exponentially increases $n_j$ by a fixed rate $1.0004$ and runs for a much smaller number of iterations set according to Theorem \ref{thm:dynamic_nj}, achieves a better accuracy per dollar, compared with using 1 worker for $J=10000$ iterations ({\bf Static $n=1$}).

\captionsetup[figure]{labelfont=bf}
\begin{figure}
    \centering
    \begin{subfigure}[t]{0.48\linewidth}
    \includegraphics[width=\textwidth]{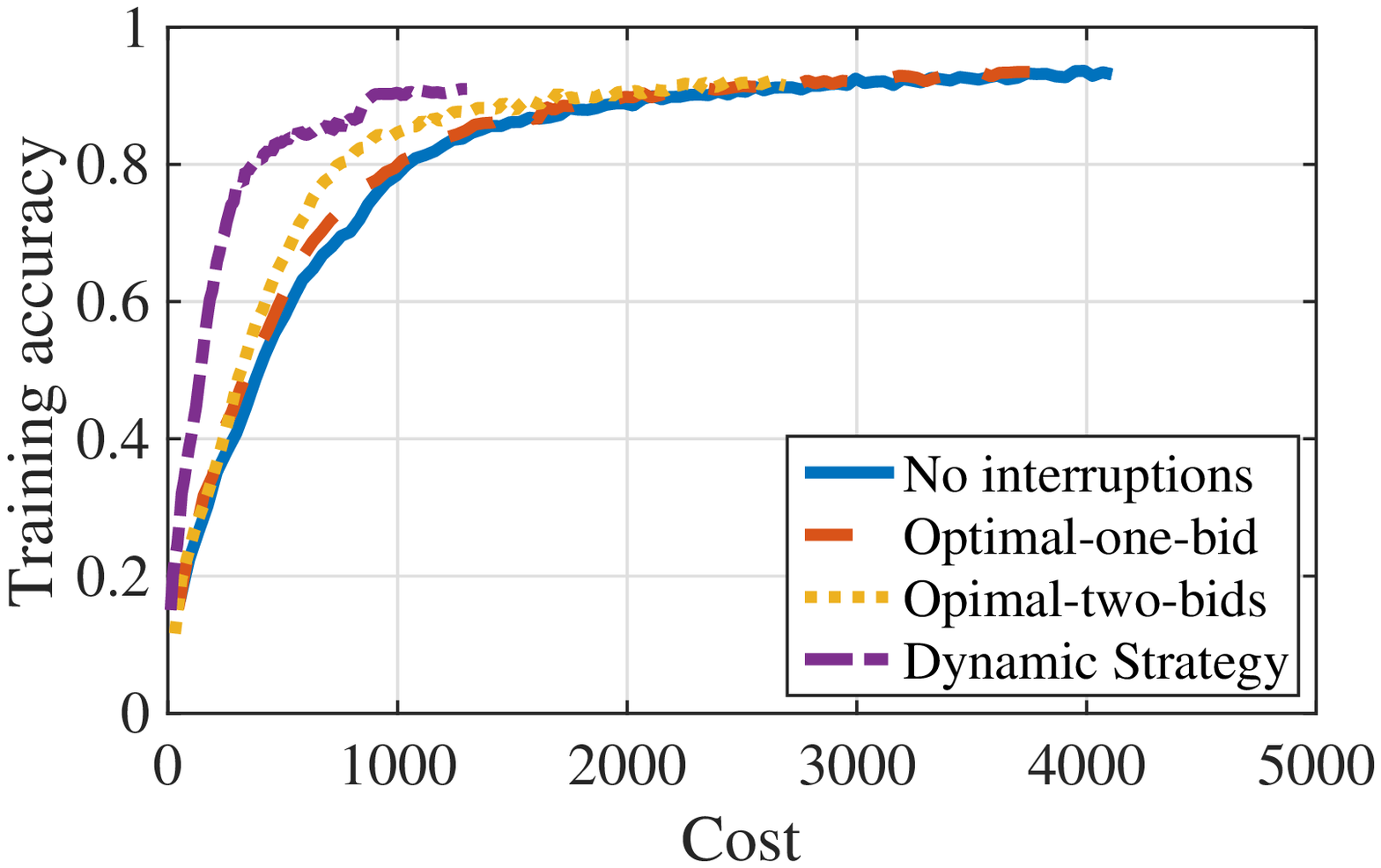}
    \caption{Accuracy-vs-cost, uniform spot price distribution}
    \label{fig:cif-acc-cost-uniform}
    \end{subfigure}
    \hspace{0.01cm}
    \begin{subfigure}[t]{0.48\linewidth}
    \includegraphics[width=\textwidth]{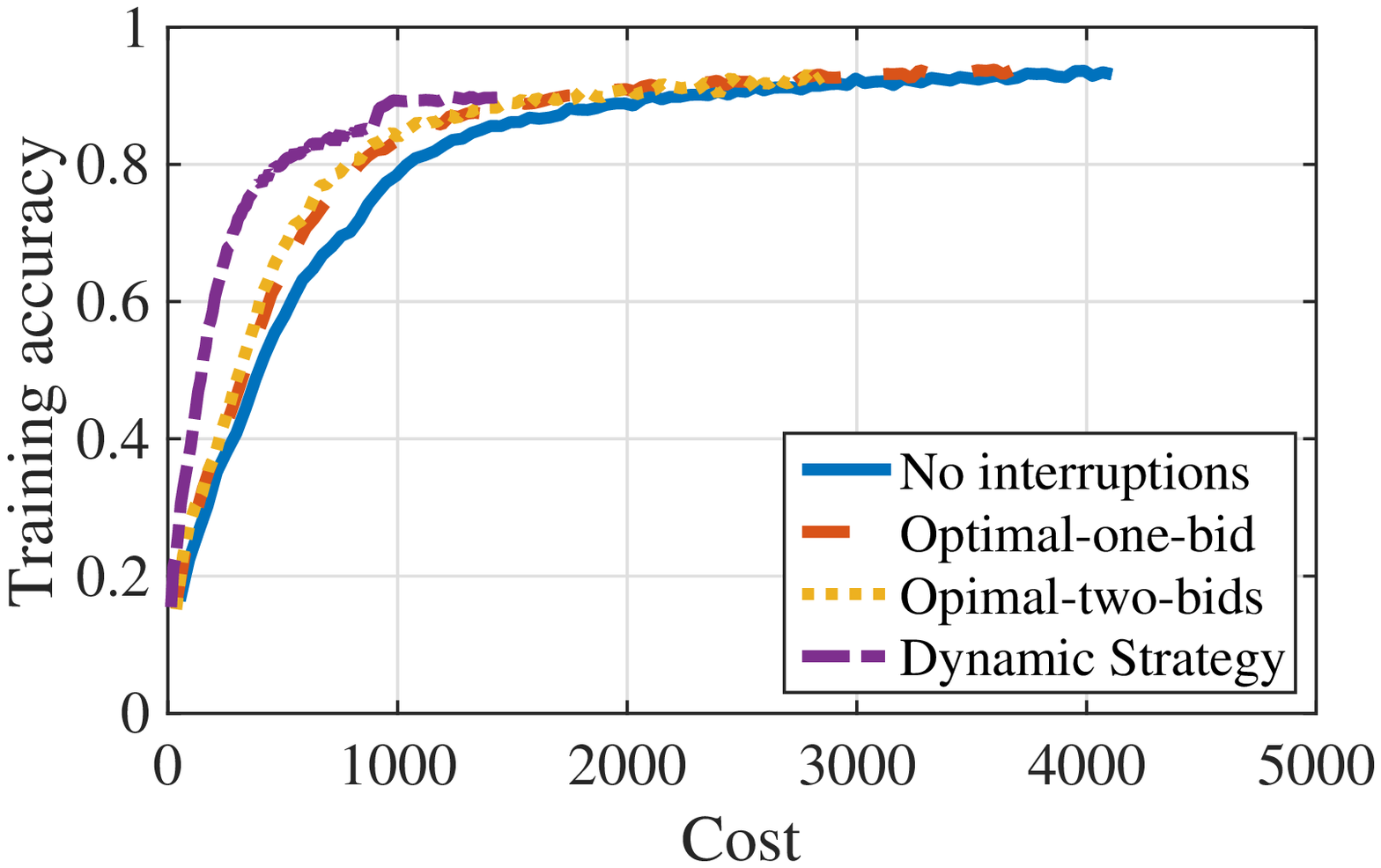}
    \caption{Accuracy-vs-cost, Gaussian spot price distribution}
    \label{fig:cif-acc-cost-normal}
    \end{subfigure}
    \begin{subfigure}[t]{0.48\linewidth}
    \includegraphics[width=\textwidth]{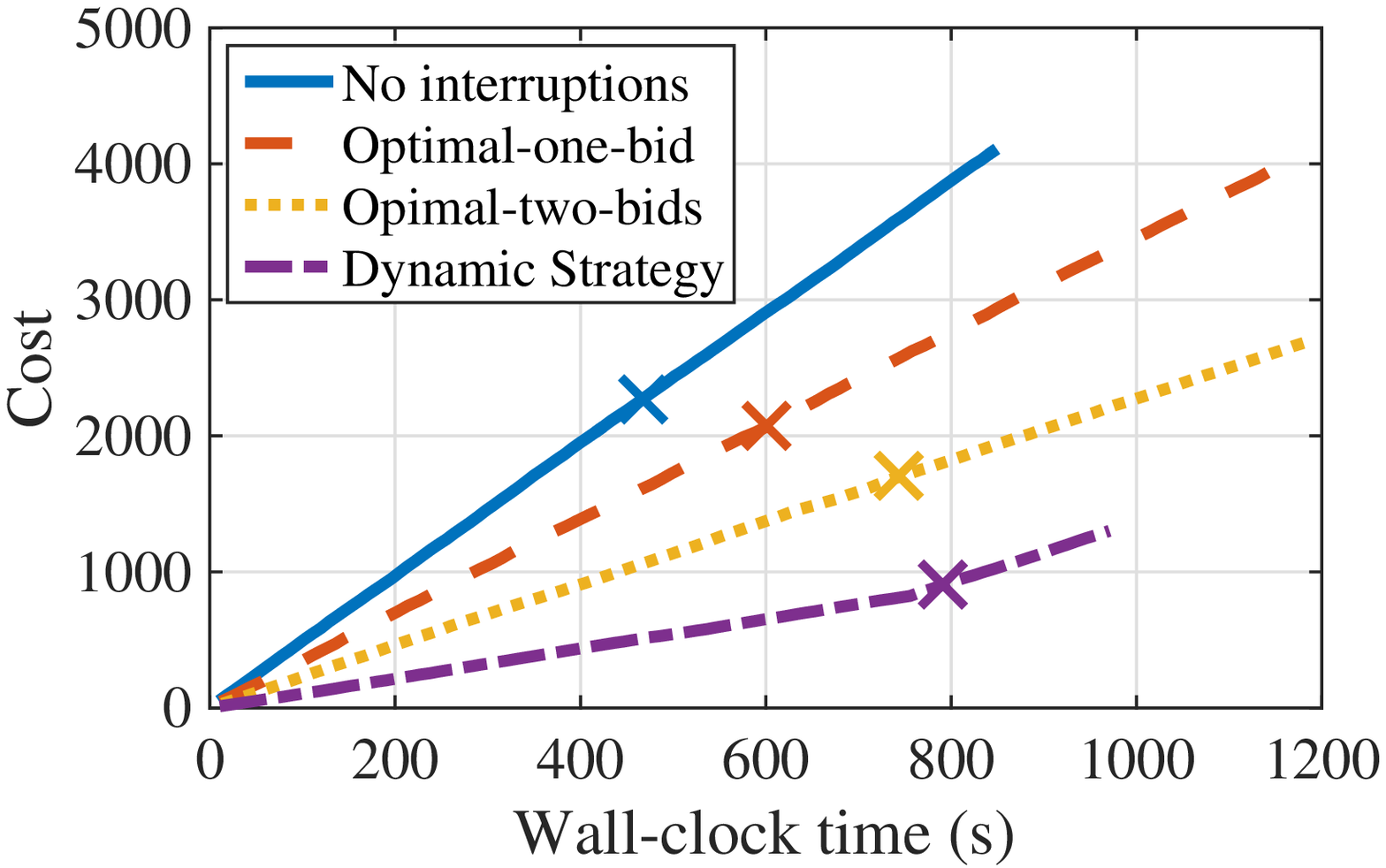}
    \caption{Cost-vs-time, uniform spot price distribution}
    \label{fig:cif-cost-uniform}
    \end{subfigure}
    \hspace{0.01cm}
    \begin{subfigure}[t]{0.48\linewidth}
    \includegraphics[width=\textwidth]{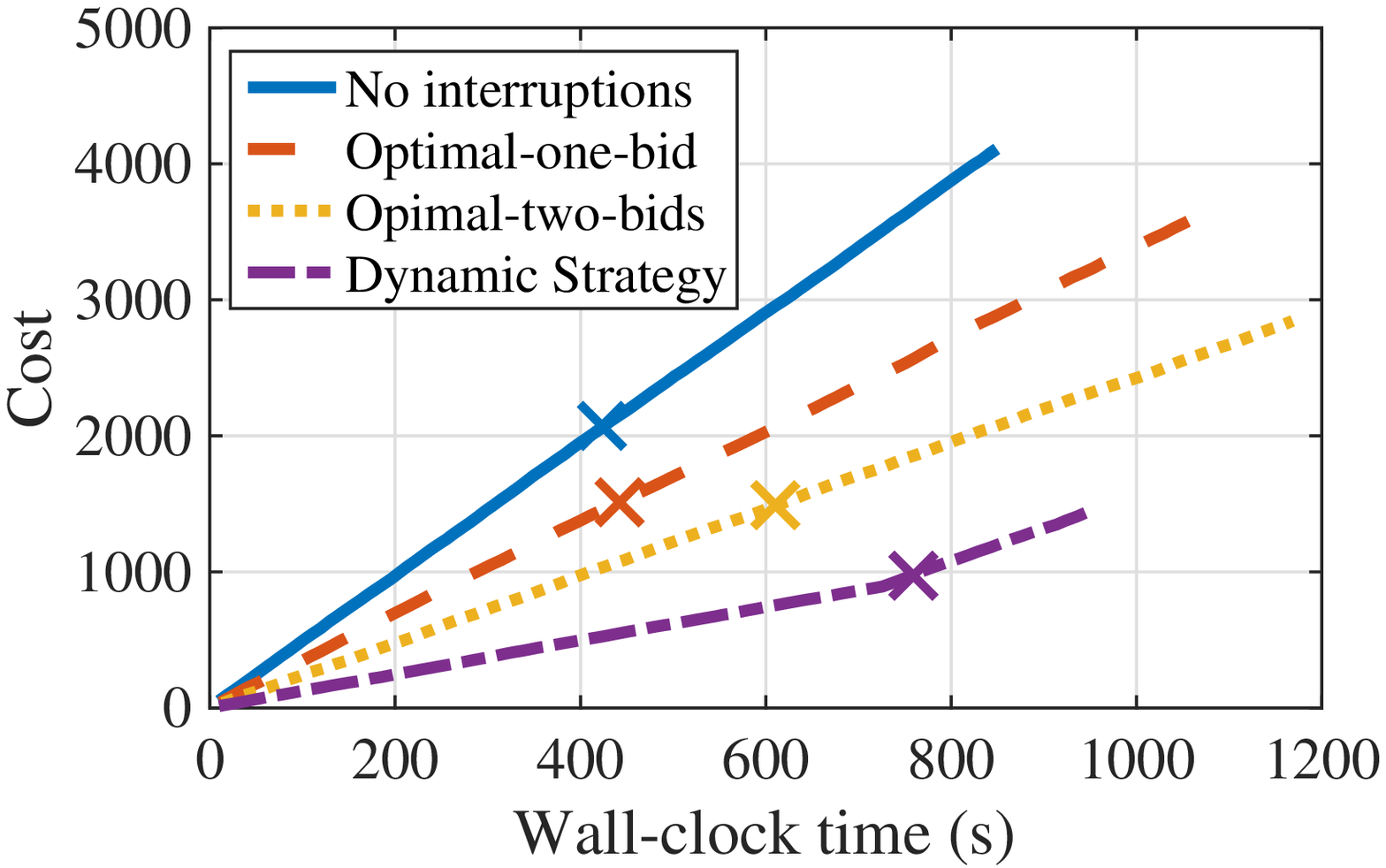}
    \caption{Cost-vs-time, Gaussian spot price distribution}
    \label{fig:cif-cost-normal}
    \end{subfigure}
    \caption{The dynamic strategy (\subref{fig:cif-acc-cost-uniform},\subref{fig:cif-acc-cost-normal}) achieves the highest test accuracy under any given cost under synthetic spot prices. The markers on the curves in (\subref{fig:cif-cost-uniform},\subref{fig:cif-cost-normal}) show the cost when achieving a $98\%$ test accuracy; at which point No-interruptions, Optimal-one-bid, and Optimal-two-bids respectively increase the cost by 134\%, 82\%, 46\% under the uniform distribution, and 103\%, 101\%, 43\% under the Gaussian distribution relative to the dynamic strategy. }
    \label{fig:cif}
\end{figure}

\captionsetup[figure]{labelfont=bf}
\begin{figure}
    \centering
    \begin{subfigure}[t]{0.47\linewidth}
    \includegraphics[width=\textwidth]{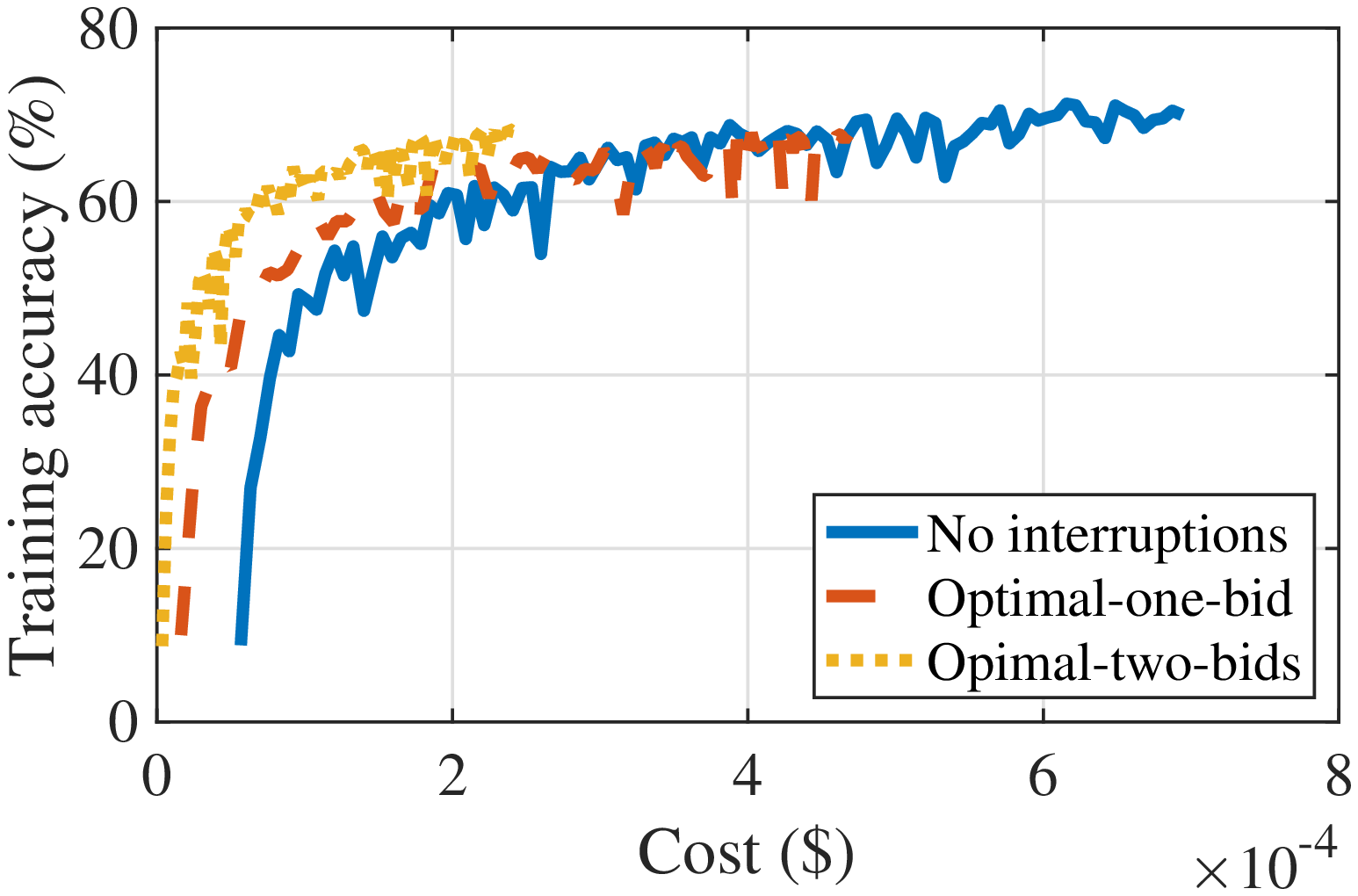}
    \caption{Accuracy-vs-cost}
    \label{fig:data_cif_acc_cost}
    \end{subfigure}
    \hspace{0.03cm}
    \begin{subfigure}[t]{0.47\linewidth}
    \includegraphics[width=\textwidth]{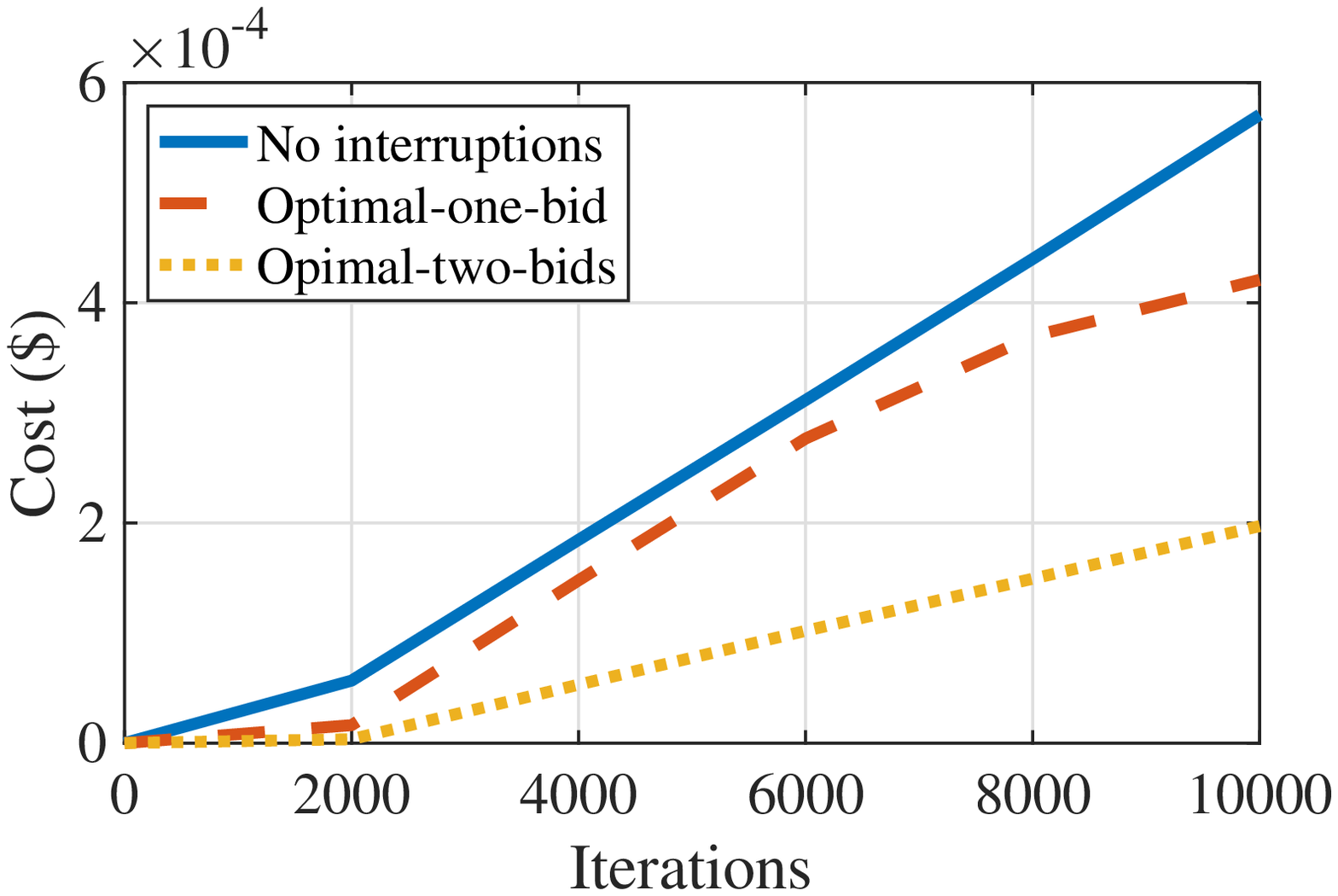}
    \caption{Cost-vs-iterations}
    \label{fig:data_cif_cost_time}
    \end{subfigure}
    \caption{Under historical price traces of the c5x.large spot instances in the region of us-west-2a (Oregon), Optimal-one-bid and Optimal-two-bids can reduce the cost by 26.27\% and 65.46\% respectively compared with No-interruptions (Figure \ref{fig:data_cif_cost_time}) while achieving 96.78\% and 96.46\% of the training accuracy that No-interruptions achieves (Figure  \ref{fig:data_cif_acc_cost}).}
    \label{fig:data_cif}
\end{figure}

\captionsetup[figure]{labelfont=bf}
\begin{figure}
    \centering
    \begin{subfigure}[t]{0.47\linewidth}
    \includegraphics[width=\textwidth]{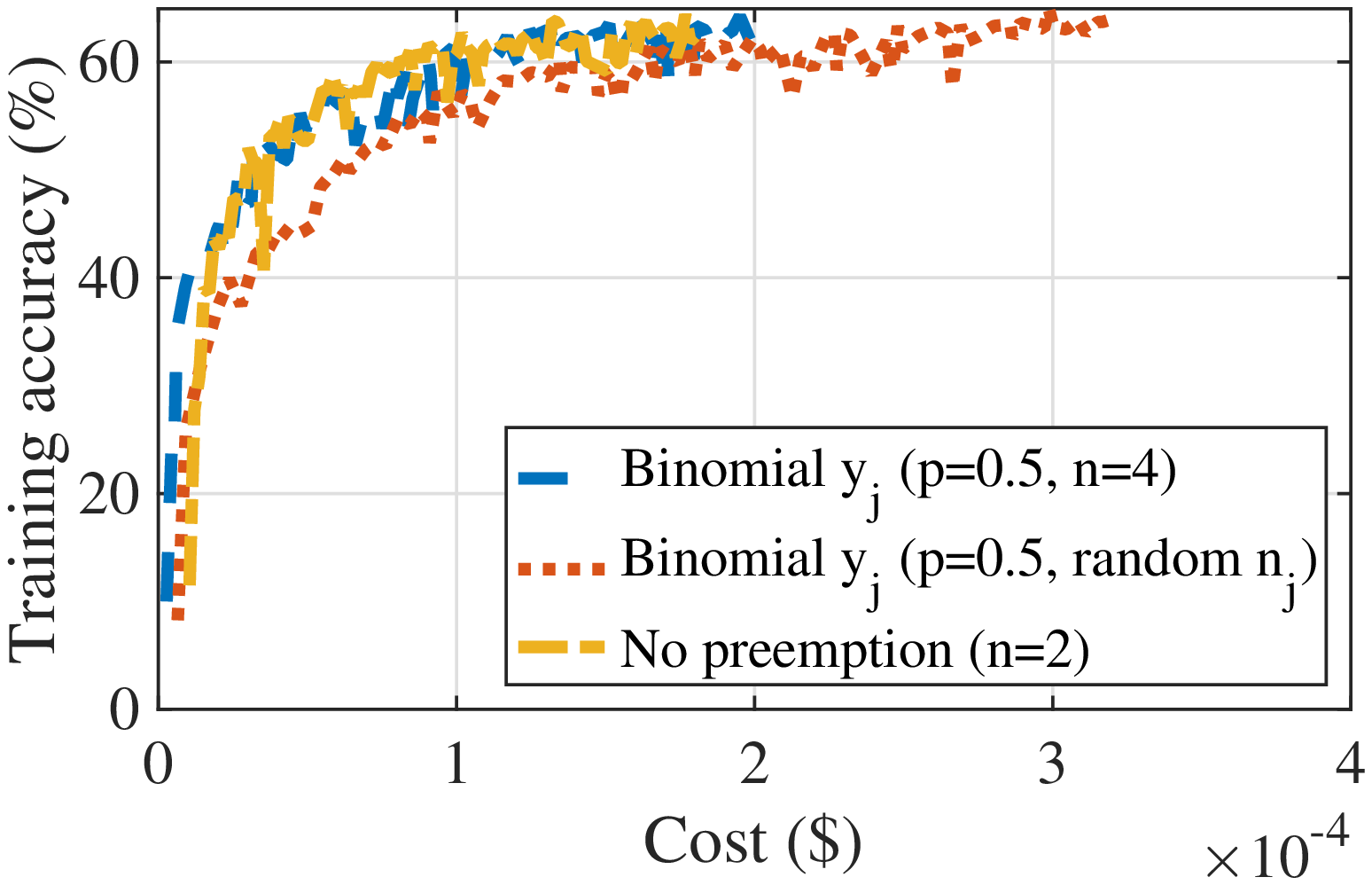}
    \caption{Accuracy-vs-cost varying preemption probability and $n$}
    \label{fig:preemption_distributions}
    \end{subfigure}
    \hspace{0.03cm}
    \begin{subfigure}[t]{0.47\linewidth}
    \includegraphics[width=\textwidth]{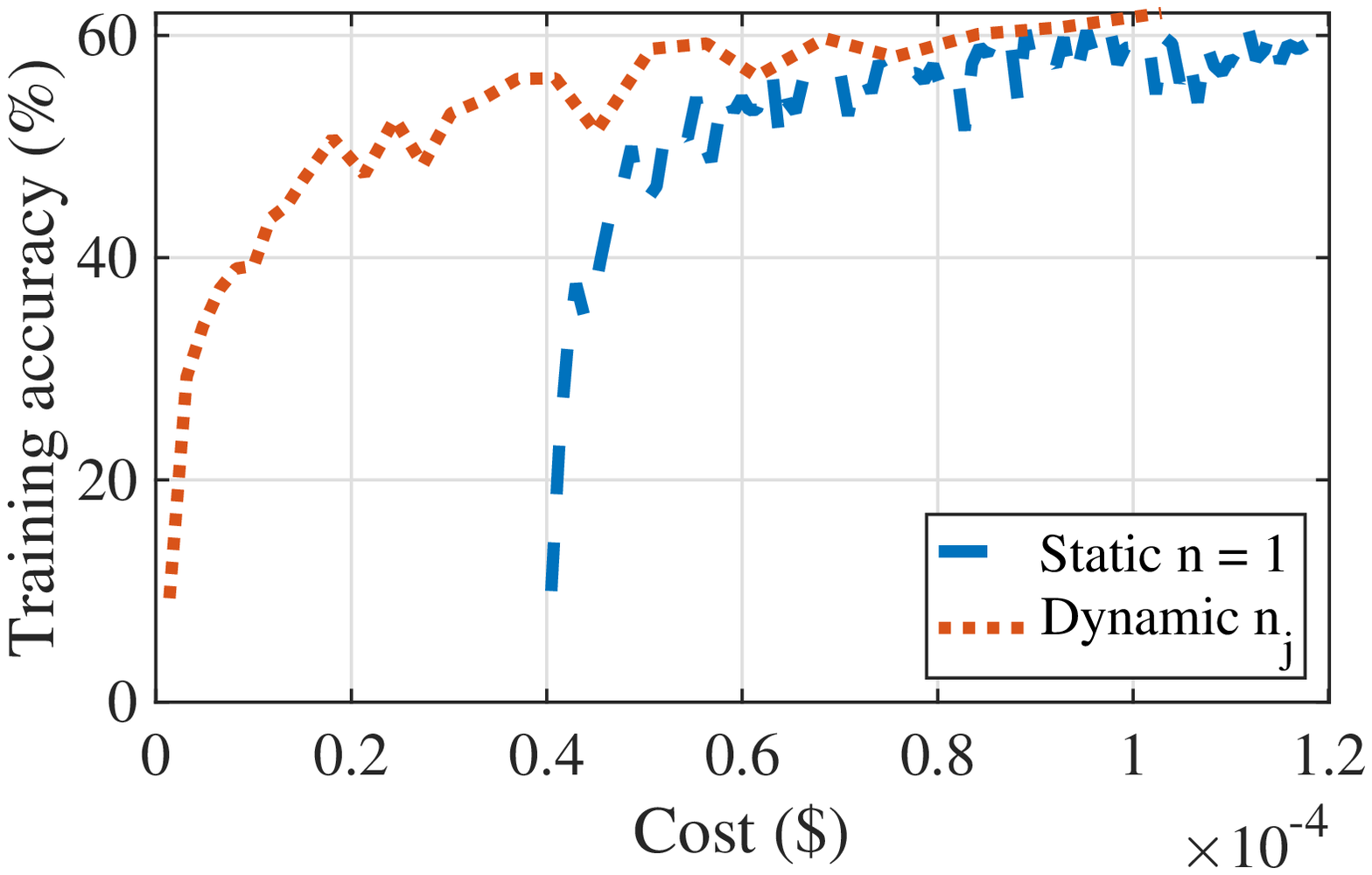}
    \caption{Accuracy-vs-cost: static-vs-dynamic strategies}
    \label{fig:preemption_strategies}
    \end{subfigure}
    \caption{Using $n$ estimated based on Theorem \ref{thm:co-opt_n_J} achieves higher accuracy per dollar than randomly setting $n$ (Figure \ref{fig:preemption_distributions}). Compared with using $1$ worker for $J = 10000$ iterations, dynamically setting $n_j=1.0004^{j-1}$ and the number of iterations according to Theorem \ref{thm:dynamic_nj} with $\chi = 1$ achieves higher accuracy per dollar on EC2 spot instances.}
    \label{fig:my_label}
\end{figure}

\section{Discussion and Conclusion}\label{sec:conclusion}
In this work, we consider the use of volatile workers that run distributed SGD algorithms to train machine learning models. We first focus on Amazon EC2 spot instances, which allow users to reduce job cost at the expense of a longer training time to achieve the same model accuracy. Spot instances allow users to choose how much they are willing to pay for computing resources, thus allowing them to control the trade off between a higher cost and a longer completion time or higher training error. We quantify these tradeoffs and derive new bounds on the training error when using time-variant numbers of workers. We finally use these results to derive optimized bidding strategies for users on spot instances and propose practical strategies for scenarios without controlling the preemption of the instances by submitting bids. We validate these strategies by comparing them to heuristics when training neural network models on the CIFAR-10 image dataset.

Our proposed strategies are an initial step towards a more comprehensive set of methods that allow distributed ML algorithms to exploit the benefits of volatile instances. As a simple extension, one might adapt the bids over time as we obtain better estimates of the iteration running time. Our bidding strategies might also be generalized to allow different bids for each worker. Even more generally, one can envision dividing a resource budget across workers, with the budget controlling each worker's availability. This budget might be a monetary budget when workers are run on cloud instances, but if the workers are instead run on mobile devices, it might instead represent a power budget that controls how often these devices can afford to process data.

\section{Acknowledgments}
This work was supported by NSF grants CNS-1751075, CNS-1909306, CCF-1850029, and a 2018 IBM Faculty Research Award. We also thank Fangjing Wu for her assistance.

\appendix
\begin{proof}[Proof of Theorem~\ref{thm:error-cv}]
$G(\mathbf{w}_{j+1})$ is at most:
\begin{align}\label{eq:bound_Lipschitz}
    G(\mathbf{w}_j) + \nabla G(\mathbf{w}_j)\cdot (\mathbf{w}_{j+1} - \mathbf{w}_j) + \frac{L}{2} || \mathbf{w}_{j+1} - \mathbf{w}_j ||_2^2
\end{align}
due to Assumption \ref{asmp-lipschitz}. Combining \eqref{eq:sync_sgd}, Assumption \ref{asmp:moments}, and \eqref{eq:bound_Lipschitz},
\begin{align}
    &\expect*{G(\mathbf{w}_{j+1}) - G(\mathbf{w}_j)} \nonumber\\
    &\leq - \alpha||\nabla G(\mathbf{w}_j)||_2^2 \left( \mu - \frac{\alpha L M_G}{2}  \right) + \expect*{\frac{\alpha^2LM}{2y_j}}\\
    &\leq - \frac{1}{2}\alpha\mu||\nabla G(\mathbf{w}_j)||_2^2 + \expect*{\frac{\alpha^2LM}{2y_j}} \label{eq:bnd_2},
\end{align}
where \eqref{eq:bnd_2} follows from our choice of $0< \alpha< \frac{\mu}{LM_G}$. 
If $G(\cdot)$ is $c$-strong convex with $c\le L$, then it satisfies the Polyak-Lojasiewicz condition $|| \nabla G(\mathbf{w}_j) ||_2^2 \ge 2c \left( G(\mathbf{w}_j) - G^* \right), \forall \mathbf{w}_j$ (Appendix B of \cite{pl-condition}). Substituting this into \eqref{eq:bnd_2} and subtracting $G^*$ on both sides, we have:
%
%
\begin{align*}
    \expect*{G(\mathbf{w}_{j+1})} &\leq  (1- \alpha c \mu)\left( G(\mathbf{w}_j) - G^*\right) + \expect*{\frac{\alpha^2 L M}{2y_j}}
\end{align*}
Applying the above inequality recursively over all iterations leads to \eqref{eq:error-cv}, and the theorem follows.
\end{proof}

\begin{proof}[Proof of Lemma \ref{lem:single-bid-cost-monotonicity}]
\label{sec:lem:single-bid-cost-monotonicity}
The objective function \eqref{eq:simp-obj-cost} takes the sum of price multipled by the runtime over all $J$ iterations with at least one active worker.
Therefore, we have $\expect*{C}=\frac{J \int_{\ubar{p}}^{b} n \expect*{R(n)} p f(p) \text{d} p}{F(b)}$, which equals $\frac{J n \expect*{R(n)}}{F(b)} \int_{\ubar{p}}^b \left(\left( p F(p) \right)^{\prime} - F(p)\right) \text{d}p$ and thus $\frac{J n \expect*{R(n)}}{F(b)}\left( b F(b) - bF(\ubar{p}) -\int_{\ubar{p}}^b F(b)\text{d}p \right)$.
The lemma follows as $F(\ubar{p}) = 0$. 
%
\end{proof}

\label{sec:thm:identical-opt-bids}
\begin{proof}[Proof of Theorem~\ref{thm:identical-opt-bids}]
Note that $\expect*{C}$ is non-increasing with $b$, the optimal number of iterations equals $\hat{\phi}^{-1}(\epsilon)$, and the expected cost in non-decreasing with $b$, the optimal bid price has $\expect*{\tau} = \theta$.
Setting the right-hand side of \eqref{eq:expected-completion-time} to be equal to $\theta$ and taking $J = \hat{\phi}^{-1}(\epsilon)$, we can conclude that the optimal $b$ should be equal to 
$F^{-1}\left( \frac{\hat{\phi}^{-1}(\epsilon)\expect*{R(n)}}{\theta} \right)$.
\end{proof}

\begin{proof}[Proof of Theorem \ref{thm:co-opt_n_J}]
    Given that $y_j$ is i.i.d.~across all iterations with $\expect*{y_j | y_j > 0}\propto n$, it suffices to minimize $J\cdot n$ subject to $A\beta^J + \frac{B\left(1-\beta^J\right)}{n(1-\beta)} \le \epsilon$. Suppose the $n^*$ is a feasible solution that is not least integer that makes the error constraint tight, \ie, satisfying $A\beta^{J^*} + \frac{B\left(1-\beta^{J^*}\right)}{(n^*-1)(1-\beta)} \le  \epsilon$, there exists a feasible solution $n' = n^* - 1$ such that the objective value $J^*\cdot n'$ is strictly smaller than $J^* \cdot n^*$, a contradiction. Therefore, we can replace the objective function $J\cdot n$ by $\frac{BJ(1-\beta^J)}{(1-\beta)(\epsilon - A\beta^J)}$. Letting its derivative to be zero leads to $\frac{A\beta^{\tilde{J}}\left( \tilde{J}\ln{\frac{1}{\beta}} + 1 - \beta^{\tilde{J}} \right)}{1 + \beta^{\tilde{J}}(\tilde{J}\ln{\frac{1}{\beta}} - 1)} (\text{denoted by}~H(\tilde{J}))= \epsilon$ where $\tilde{J}$ can be fractional. One can verify that $H(\tilde{J})$ monotonically decreases with $\tilde{J}$ and the objective function is smooth. Thus, $J^*$ should be among: the least integer no smaller than $\tilde{J}$, the largest integer no larger than $\tilde{J}$, and $\floor*{\theta \delta}$, whichever that yields the smallest objective value, the theorem follows.
\end{proof}

\begin{proof}[Proof of Theorem \ref{thm:dynamic_nj}]
Based on our Theorem \ref{thm:error-cv}, the error bound of using $\ceil*{n_0\eta^{j-1}}$ workers in iteration $j$ and running the SGD for $J'$ iterations is at most:
\begin{align}
\label{eq:bound_dynamic_nj}
    &(1-\alpha c \mu)^{J'} \expect*{G(\mathbf{w}_0)} 
    + B \sum_{j=1}^{J'} \frac{(1-\alpha c \mu)^{J'-j}}{\left(n_0 \eta^{j-1}\right)^{\chi}}\notag\\
    =& (1-\alpha c \mu)^{J'} \expect*{G(\mathbf{w}_0)} 
    + \frac{B}{n_0^{\chi}} \cdot \sum_{j=1}^{J'} \frac{(1-\alpha c \mu)^{J'-1}}{\left[ \eta^{\chi} (1-\alpha c \mu) \right]^{j-1}} \notag\\
    =& (1-\alpha c \mu)^{J'} \expect*{G(\mathbf{w}_0)} 
    + \frac{B}{n_0^{\chi}} \cdot (1-\alpha c \mu)^{J'-1} \cdot \frac{1 - x^{J'}}{1 - x},
\end{align}
where we define $x=\frac{1}{\eta^{\chi} (1-\alpha c \mu)}$ and $B$ is a constant linear with $\frac{\alpha^2 L M}{2}$. Given our choice of $\eta^{\chi} > (1-\alpha c \mu)^{-1}$, the error bound will exponentially decrease with $J'$. In comparison, if using $n_0$ workers for $J$ iterations, the error is at most:
\begin{align}
\label{eq:bound_static_n}
    (1-\alpha c \mu)^{J} \expect*{G(\mathbf{w}_0)} 
    + \frac{B}{n_0} \cdot \frac{1-(1-\alpha c \mu)^J}{\alpha c \mu}
\end{align}
Based on \eqref{eq:bound_dynamic_nj}, \eqref{eq:bound_static_n}, and our choice of $\eta$, the error decay rate is no smaller than $(1-\alpha c \mu)$ in the dynamic strategy (bound \eqref{eq:bound_dynamic_nj}) and equals $(1-\alpha c \mu)$ in the static strategy. Moreover, when $J \to +\infty$, the error bound of the dynamic strategy approaches $\frac{B\beta^{J' - 1} \left(1 - (\frac{1}{\beta \eta^{\chi}})^{J'}\right)}{n_0^{\chi}\left(1 - \frac{1}{\beta \eta^{\chi}}\right)}$, where $\beta := 1-\alpha c \mu$, while that of the static strategy \eqref{eq:bound_static_n} approaches $\frac{B}{(1 - \beta)n_0}$. Putting $J' = \log_{\eta} \left( 1+ (\eta -1)J\right)$ into the former, it becomes $\frac{B\left[ (\eta^{\chi} + 1)J + 1 \right]^{\log_{\eta^{\chi}} \beta}}{n \beta (1 -\frac{1}{\beta \eta^{\chi}})}$ which is smaller than $\frac{B}{(1 -\beta)n_0}$ (error bound of the static strategy) when $J$ is sufficiently large due to $\log_{\eta^{\chi}} \beta < 0$, the theorem follows.
\end{proof}

\begin{proof}[Proof of Lemma \ref{lem:distributions_preemption}]
For such a uniform $y_j$, we have:
\begin{align*}
    \expect*{\frac{1}{y_j}} = \sum_{k=1}^{n_j} \frac{1}{k} \cdot \frac{1}{n_j} \le \frac{\ln n_j + 1}{n_j} \le O\left(\frac{1}{n_j^{1/2}}\right)
\end{align*}
If each worker is preempted with probability $q$, it suffices to show that for a constant $d>0$, any $q\in [\frac{1}{2}, 1)$, and $\gamma \in (0,1)$, $\left|\expect*{\frac{1}{y_j}} - \expect*{\frac{1}{y_j + 1}}\right| \leq dn^{-\gamma}$ is at most
\begin{align*}
     &\leq \frac{1}{1-q^n}\left(\sum_{y = 1}^{n^{\gamma}} \frac{1}{y(y + 1)}{n\choose y}q^n + \sum_{y = n^{\gamma} + 1}^{n} \frac{1}{y(y + 1)}{n\choose y}q^n \right)\\
     &\leq \frac{1}{1-q^n}\left( n^{\gamma}\left(qn^{n^{\gamma - 1}}\right)^{n} + \frac{1}{n^{2\gamma - 1}}\right) \leq \frac{d}{n^{2\gamma - 1}}
 \end{align*}
 and $\expect*{\frac{1}{y_j + 1}} = \frac{1-q^{n+1}}{(1+n)(1-q)} $ according to \cite{inv_binomial}. The result also holds for $q\in (0, \frac{1}{2})$ by applying the derivation on $1-q$ which is $\in [\frac{1}{2}, 1)$), rather than on $q$, the lemma follows.
\end{proof}

\balance
\bibliographystyle{IEEEtran}
\bibliography{main.bib}

\begin{thebibliography}{10}
\providecommand{\url}[1]{#1}
\csname url@samestyle\endcsname
\providecommand{\newblock}{\relax}
\providecommand{\bibinfo}[2]{#2}
\providecommand{\BIBentrySTDinterwordspacing}{\spaceskip=0pt\relax}
\providecommand{\BIBentryALTinterwordstretchfactor}{4}
\providecommand{\BIBentryALTinterwordspacing}{\spaceskip=\fontdimen2\font plus
\BIBentryALTinterwordstretchfactor\fontdimen3\font minus
  \fontdimen4\font\relax}
\providecommand{\BIBforeignlanguage}[2]{{%
\expandafter\ifx\csname l@#1\endcsname\relax
\typeout{** WARNING: IEEEtran.bst: No hyphenation pattern has been}%
\typeout{** loaded for the language `#1'. Using the pattern for}%
\typeout{** the default language instead.}%
\else
\language=\csname l@#1\endcsname
\fi
#2}}
\providecommand{\BIBdecl}{\relax}
\BIBdecl

\bibitem{SGD1951}
H.~Robbins and S.~Monro, ``A stochastic approximation method,'' \emph{The
  Annals of Mathematical Statistics}, vol.~22, no.~3, pp. 400--407, 1951.

\bibitem{sgd-ls}
L.~Bottou, ``Large-scale machine learning with stochastic gradient descent,''
  in \emph{Proceedings of COMPSTAT}, 2010.

\bibitem{parallel-training-ls-distbelief}
J.~D. et~al., ``Large scale distributed deep networks,'' in \emph{International
  Conference on Neural Information Processing Systems (NIPS)}, vol.~1, 2012,
  pp. 1223--1231.

\bibitem{spot}
{Amazon EC2}, ``Amazon ec2 spot instances,''
  \url{https://aws.amazon.com/ec2/spot/}, 2019.

\bibitem{gcp}
{Google Cloud Platform}, ``Preemptible virtual machines,''
  \url{https://cloud.google.com/preemptible-vms/}, 2019.

\bibitem{low_priority_azure}
{Microsoft Azure}, ``Announcing low-priority vms on scale sets now in public
  preview,''
  \url{https://azure.microsoft.com/en-us/blog/low-priority-scale-sets/}, 2018.

\bibitem{aws-fleet}
{Amazon EC2}, ``Spot price overrides,''
  \url{https://docs.aws.amazon.com/AWSEC2/latest/UserGuide/spot-fleet.html#spot-price-overrides},
  2019.

\bibitem{yang2016zccloud}
F.~Yang and A.~A. Chien, ``Zccloud: Exploring wasted green power for
  high-performance computing,'' in \emph{2016 IEEE International Parallel and
  Distributed Processing Symposium (IPDPS)}.\hskip 1em plus 0.5em minus
  0.4em\relax IEEE, 2016, pp. 1051--1060.

\bibitem{chien2016characterizing}
A.~A. Chien, F.~Yang, and C.~Zhang, ``Characterizing curtailed and uneconomic
  renewable power in the mid-continent independent system operator,''
  \emph{arXiv preprint arXiv:1702.05403}, 2016.

\bibitem{konevcny2016federated}
J.~Kone{\v{c}}n{\`y}, H.~B. McMahan, F.~X. Yu, P.~Richt{\'a}rik, A.~T. Suresh,
  and D.~Bacon, ``Federated learning: Strategies for improving communication
  efficiency,'' \emph{arXiv preprint arXiv:1610.05492}, 2016.

\bibitem{tao2018esgd}
Z.~Tao and Q.~Li, ``esgd: Communication efficient distributed deep learning on
  the edge,'' in \emph{$\{$USENIX$\}$ Workshop on Hot Topics in Edge Computing
  (HotEdge 18)}, 2018.

\bibitem{BidCloud}
L.~Zheng, C.~Joe-Wong, C.~W. Tan, M.~Chiang, and X.~Wang, ``How to bid the
  cloud,'' in \emph{Proc. ACM SIGCOMM}, 2015.

\bibitem{cifar10}
A.~Krizhevsky, V.~Nair, and G.~Hinton, ``The cifar-10 dataset,''
  \url{https://www.cs.toronto.edu/~kriz/cifar.html}.

\bibitem{NotBidCloud}
P.~Sharma, D.~Irwin, and P.~Shenoy, ``How not to bid the cloud,'' in
  \emph{Proc. USENIX Conference on Hot Topics in Cloud Computing (HotCloud)},
  2016.

\bibitem{mini-batch}
O.~S. Ofer~Dekel, Ran Gilad-Bachrach and L.~Xiao., ``Optimal distributed online
  prediction using mini-batches.'' \emph{Journal of Machine Learning Research},
  vol.~13, no.~1, pp. 165--202, 2012.

\bibitem{shamir2014communication}
O.~Shamir, N.~Srebro, and T.~Zhang, ``Communication-efficient distributed
  optimization using an approximate newton-type method,'' in
  \emph{International conference on machine learning}, 2014, pp. 1000--1008.

\bibitem{kamp2018efficient}
M.~Kamp, L.~Adilova, J.~Sicking, F.~H{\"u}ger, P.~Schlicht, T.~Wirtz, and
  S.~Wrobel, ``Efficient decentralized deep learning by dynamic model
  averaging,'' \emph{arXiv preprint arXiv:1807.03210}, 2018.

\bibitem{McMahan2017Federated}
\BIBentryALTinterwordspacing
H.~B. McMahan, E.~Moore, D.~Ramage, S.~Hampson, and B.~A. y~Arcas,
  ``Communication-efficient learning of deep networks from decentralized
  data,'' in \emph{Proceedings of AISTATS}, 2017. [Online]. Available:
  \url{http://arxiv.org/abs/1602.05629}
\BIBentrySTDinterwordspacing

\bibitem{Staleness:Aistats}
S.~Dutta, G.~Joshi, S.~Ghosh, P.~Dube, and P.~Nagpurkar, ``Slow and stale
  gradients can win the race: Error-runtime trade-offs in distributed sgd,'' in
  \emph{Proceedings of AISTATS}, 2018.

\bibitem{Bottou}
L.~Bottou, F.~E. Curtis, and J.~Nocedal, ``Optimization methods for large-scale
  machine learning,'' \emph{SIAM Review}, vol.~60, no.~2, pp. 223--311, 2018.

\bibitem{wang2019efficient}
\BIBentryALTinterwordspacing
D.~Wang, G.~Joshi, and G.~W. Wornell, ``Efficient straggler replication in
  large-scale parallel computing,'' \emph{ACM Trans. Model. Perform. Eval.
  Comput. Syst.}, vol.~4, no.~2, Apr. 2019. [Online]. Available:
  \url{https://doi.org/10.1145/3310336}
\BIBentrySTDinterwordspacing

\bibitem{spot-ddl}
M.~Zafer, Y.~Song, and K.-W. Lee, ``Optimal bids for spot vms in a cloud for
  deadline constrained jobs,'' in \emph{Proc. of IEEE CLOUD}, 2012.

\bibitem{proteus}
A.~Harlap, A.~Tumanov, A.~Chung, G.~R. Ganger, and P.~B. Gibbons, ``Proteus:
  Agile ml elasticity through tiered reliability in dynamic resource markets,''
  in \emph{Proc. of European Conference on Computer Systems}, 2017.

\bibitem{lee2017deepspotcloud}
K.~Lee and M.~Son, ``Deepspotcloud: leveraging cross-region gpu spot instances
  for deep learning,'' in \emph{Proceedings of IEEE CLOUD}.\hskip 1em plus
  0.5em minus 0.4em\relax IEEE, 2017, pp. 98--105.

\bibitem{yan2016tr}
Y.~Yan, Y.~Gao, Y.~Chen, Z.~Guo, B.~Chen, and T.~Moscibroda, ``Tr-spark:
  Transient computing for big data analytics,'' in \emph{Proceedings of the
  Seventh ACM Symposium on Cloud Computing}.\hskip 1em plus 0.5em minus
  0.4em\relax ACM, 2016, pp. 484--496.

\bibitem{Boyd:cv}
S.~Boyd and L.~Vandenberghe, ``Convex optimization,'' \emph{Cambridge
  university press}, 2014.

\bibitem{stragglers-ls}
J.~Dean and L.~A. Barroso, ``The tail at scale,'' \emph{Communications of the
  ACM}, vol.~56, no.~2, pp. 74--80, 2013.

\bibitem{wiki:pdf}
\BIBentryALTinterwordspacing
``Probability density function.'' [Online]. Available:
  \url{https://en.wikipedia.org/wiki/Probability_density_function}
\BIBentrySTDinterwordspacing

\bibitem{wiki:cdf}
\BIBentryALTinterwordspacing
``Cumulative distribution function.'' [Online]. Available:
  \url{https://en.wikipedia.org/wiki/Cumulative_distribution_function}
\BIBentrySTDinterwordspacing

\bibitem{price_change}
``How spot instances work,''
  \url{https://docs.aws.amazon.com/aws-technical-content/latest/cost-optimization-leveraging-ec2-spot-instances/how-spot-instances-work.html}.

\bibitem{icml_increase_nodes}
\BIBentryALTinterwordspacing
J.~Wang and G.~Joshi, ``Adaptive communication strategies to achieve the best
  error-runtime trade-off in local-update {SGD},'' in \emph{Proc.~of SysML
  Conference}, 2019. [Online]. Available:
  \url{"https://arxiv.org/pdf/1810.08313.pdf"}
\BIBentrySTDinterwordspacing

\bibitem{vldb_increase_machines}
H.~Yun, H.-F. Yu, C.-J. Hsieh, S.~V.~N. Vishwanathan, and I.~Dhillon, ``Nomad:
  Non-locking, stochastic multi-machine algorithm for asynchronous and
  decentralized matrix completion,'' in \emph{Proc.~of VLDB Endowment}, 2014.

\bibitem{iclr_backup_workers}
J.~Chen, X.~Pan, R.~Monga, and S.~Bengio, ``Revisiting distributed synchronous
  sgd,'' in \emph{Proc.~of ICLR Workshop Track}, 2016.

\bibitem{icml_exp_batch}
H.~Yu and R.~Jin, ``On the computation and communication complexity of parallel
  sgd with dynamic batch sizes for stochastic non-convex optimization,'' in
  \emph{Proc.~of ICML}, 2019.

\bibitem{resnet}
\BIBentryALTinterwordspacing
K.~He, X.~Zhang, S.~Ren, and J.~Sun, ``Deep residual learning for image
  recognition,'' 2015. [Online]. Available:
  \url{https://arxiv.org/abs/1512.03385}
\BIBentrySTDinterwordspacing

\bibitem{cnn-nips2012}
G.~E.~H. Alex~Krizhevsky, IIya~Sutskever, ``Imagenet classification with deep
  convolutional neural networks,'' in \emph{Proceedings of NIPS}, 2012.

\bibitem{ray}
P.~Moritz, R.~Nishihara, S.~Wang, A.~Tumanov, R.~Liaw, E.~Liang, M.~Elibol,
  Z.~Yang, W.~Paul, M.~I. Jordan, and I.~Stoica, ``Ray: A distributed framework
  for emerging ai applications,'' in \emph{In Proceedings of USENIX OSDI},
  2018.

\bibitem{Tensorflow}
M.~Abadi, P.~Barham, J.~Chen, Z.~Chen, A.~Davis, J.~Dean, M.~Devin,
  S.~Ghemawat, G.~Irving, M.~Isard, M.~Kudlur, J.~Levenberg, R.~Monga,
  S.~Moore, D.~G. Murray, B.~Steiner, P.~Tucker, V.~Vasudevan, P.~Warden,
  M.~Wicke, Y.~Yu, , and X.~Zheng., ``Tensorflow: A system for large-scale
  machine learning.'' in \emph{In Proceedings of USENIX OSDI}, 2016.

\bibitem{pl-condition}
H.~Karimi, J.~Nutini, and M.~Schmidt, ``Linear convergence of gradient and
  proximal-gradient methods under the polyak-Łojasiewicz condition,'' in
  \emph{Proc.~of ECML PKDD}, 2016.

\bibitem{inv_binomial}
M.~T. Chao and W.~E. Strawderman, ``Negative moments of positive random
  variables,'' \emph{Journal of the American Statistical Association}, vol.~67,
  no. 338, pp. 429--431, 1972.

\end{thebibliography}

\end{document}